\documentclass[12pt,a4paper,oneside]{article}

\usepackage[a4paper,left=2cm,right=2cm,top=2.5cm,bottom=2.5cm]{geometry}
\usepackage[utf8]{inputenc}

\usepackage{algorithm}
\usepackage{algorithmic}
\usepackage[sort, numbers]{natbib}
\usepackage{forloop}

\usepackage{graphicx}
\usepackage{placeins} 

\usepackage{url}

\usepackage{amsmath}
\usepackage{breqn}
\usepackage{amsfonts}       
\usepackage{enumitem}
\usepackage{tabularx,booktabs} 
\newcommand{\ra}[1]{\renewcommand{\arraystretch}{#1}}
\usepackage{makecell}
\usepackage{lscape} 

\usepackage{amsthm}
\newtheorem{theorem}{Theorem}
\newtheorem{corollary}[theorem]{Corollary}
\newtheorem{lemma}[theorem]{Lemma}

\newtheorem{remark}{Remark}

\usepackage{ulem}
\usepackage{mathtools}

\usepackage[table]{xcolor}

\usepackage{hyperref} 

\newcommand{\ms}{\mathsf{S}}

\newcommand{\mb}{\mathsf{B}}
\newcommand{\umb}{\underline{\mathsf{B}}}
\newcommand{\N}{\mathbb{N}}
\newcommand{\Nhalf}{\frac{1}{2}\mathbb{N}}
\newcommand{\J}{\mathcal{J}}
\newcommand{\CJ}{\mathcal{C}_\mathcal{J}}
\newcommand{\CJB}{\mathcal{C}_{\mathcal{J}_{\mb}}}
\newcommand{\CJBk}{\mathcal{C}_{\mathcal{J}_{\mb_k}}}
\newcommand{\CJBkm}{\mathcal{C}_{\mathcal{J}_{\mb_{k-1}}}}
\newcommand{\R}{\mathbb{R}}

\newcommand{\tk}{\theta_{k}}
\newcommand{\tkh}{\theta_{k+\frac{1}{2}}}
\newcommand{\tkp}{\theta_{k+1}}
\newcommand{\tkm}{\theta_{k-1}}
\newcommand{\gk}{\gamma_{k}}

\newcommand{\nJ}{\nabla\J}

\newcommand{\dtk}{\Delta\tk}

\newcommand{\dgk}{\Delta g_k}
\newcommand{\dgkh}{\Delta g_{\mb_k}}
\newcommand{\dtkh}{\Delta \theta_{\mb_k}}

\newcommand{\esp}[1]{\mathbb{E}\left[#1\right]}

\newcommand{\espcondTot}[1]{\mathbb{E}\left[#1\middle|\umb_{k-1}\right]}

\begin{document}

\title{Second-order step-size tuning of SGD for non-convex optimization}
\author{\textbf{Camille Castera}\footnote{Corresponding author: \texttt{camille.castera@protonmail.com}}\\
    CNRS - IRIT,\\ Universit\'e de Toulouse, \\
    Toulouse, France
    \and
    \textbf{J\'er\^ome Bolte}$^\dagger$ \\
    Toulouse School of Economics\\ Université de Toulouse\\
    Toulouse, France
    \and
    \textbf{Cédric Févotte}$^\dagger$ \\
    CNRS - IRIT,\\ Universit\'e de Toulouse, \\
    Toulouse, France
    \and
    \textbf{Edouard Pauwels}$^\dagger$ \\
    CNRS - IRIT,\\ Universit\'e de Toulouse, \\
    DEEL, IRT Saint Exupery\\
    Toulouse, France
}
\date{}

    \maketitle
        
    \renewcommand*{\thefootnote}{$^\dagger$}
    \footnotetext[1]{Last three authors are listed in alphabetical order. }
    \renewcommand*{\thefootnote}{\arabic{footnote}}
    \setcounter{footnote}{0} 
    \begin{abstract}
    In view of a direct and simple improvement of vanilla SGD, this paper presents a fine-tuning of its step-sizes in the mini-batch case. For doing so, one estimates curvature, based on a local quadratic model and using only noisy gradient approximations. One obtains a new stochastic first-order method (Step-Tuned SGD), enhanced by second-order information, which can be seen as a stochastic version of the classical Barzilai-Borwein method. Our theoretical results ensure almost sure convergence to the critical set and we provide
    convergence rates. Experiments on deep residual network training illustrate the favorable properties of our approach. For such networks we observe, during training, both a sudden drop of the loss and an improvement of test accuracy at medium stages, yielding better results than SGD, RMSprop, or ADAM.
    \end{abstract}

    \section{Introduction}
    In the recent years, machine learning has generated a growing need for methods to solve non-convex optimization problems. In particular, the training of deep neural networks (DNNs) has received tremendous attention. Designing methods for this purpose is particularly difficult as one deals with both expensive function evaluations and limited storage capacities. This explains why stochastic gradient descent (SGD) remains the central algorithm in deep learning (DL). 
    It consists in the iterative scheme,
    \begin{equation}
        \tkp = \tk - \gamma_k \nabla\J_{\mb_k}(\theta_k),
    \end{equation}
    where $\J$ is the function to minimize (usually the empirical loss) parameterized by $\theta\in\R^P$ (the weights of the DNN), $\nabla\J_{\mb_k}(\theta_k)$ is a stochastic estimation of the gradient of $\J$ (randomness being related to the sub-sampled mini-batch ${\mb_k}$), and $\gamma_k>0$ is a step-size whose choice is critical in terms of empirical performance.
    
    In order to improve the basic SGD method, a common practice is to use adaptive methods \citep{duchi2011adaptive,tieleman2012lecture,kingma2014adam}. They act as preconditioners, reducing the importance of the choice of the sequence of step-sizes $(\gamma_k)_{k\in\N}$. 
   This paper focuses instead exclusively on the step-size issue: how can we take advantage of curvature information of non-convex landscapes in a stochastic context in order to design step-sizes adapted to each iteration?

    Our starting point to answer this question is an infinitesimal second-order variational model along the gradient direction. The infinitesimal feature is particularly relevant in DL since small steps constitute standard practice in training due to sub-sampling noise. Second-order information is approximated with first-order quantities using finite differences.  
    In deterministic (full-batch) setting, our method  corresponds to a non-convex version of the  Barzilai-Borwein (BB) method \citep{barzilai1988two,dai2002modified,XIAO20102986,Biglari2013ScalingOT} and is somehow a discrete non-convex adaption of the continuous gradient system in \citet{alvarez2004steepest}. It is also close to earlier work \cite{raydan1997barzilai},  with the major difference that our algorithm is supported by a variational model. This is essential to generalize the method to accommodate noisy gradients,  providing a convexity test similar to those in \citet{kafaki2013,curtis2016handling}.
    
    Our main contribution is a fine step-size tuning method which accelerates stochastic gradient algorithms, it is built on a strong geometrical principle: step-sizes are deduced from a carefully derived discrete approximation of a curvature term of the expected loss. We provide general convergence guarantees to critical points and rates of convergence. Extensive computations on DL problems show that our method is particularly successful for residual networks \citep{he2015}. In that case, we observe a surprising phenomenon: in the early training stage the method shows similar performances to standard DL algorithms (SGD, ADAM, RMSprop), then at medium stage, we observe simultaneously a sudden drop of the training loss and a notable increase in test accuracy.

    To summarize, our contributions are as follows:\\
    --\ Exploit the vanishing step-size nature of DL training to use infinitesimal second-order optimization for fine-tuning the step-size at each iteration. \\
    --\ Use our geometrical perspective to discretize and adapt the method to noisy gradients  despite strong non-linearities.  \\
    --\ Prove the convergence of the proposed algorithm and provide rates of convergence for twice-differentiable non-convex functions either under boundedness assumption or under Lipschitz-continuity assumptions (see Theorem~\ref{thm::mainres} and Corollary~\ref{cor::globLipschitz}).
    \\
    --\ Show that our method has remarkable practical properties, in particular when training residual networks in DL, for which one observes an advantageous ``drop down" of the loss during the training phase.

    \paragraph{Structure of the paper.} A preliminary deterministic (i.e., full-batch) algorithm is derived in Section~\ref{sec::deterministic}. We then build a stochastic mini-batch variant in Section~\ref{sec::mini-batch}, which is our core contribution. Theoretical results are stated in Section~\ref{sec::theory} and DL experiments are conducted in Section~\ref{sec::numDL}. 
    
    \section{Related work\label{sec::relatedwork}}
    
    Methods using second-order information for non-convex optimization have been actively studied in the last years, both for deterministic  and stochastic applications, see, e.g.,  \citet{royer2017complexity,carmon2017convex,allenzhu2017natasha,krishnan2017neumann,MartensG15,liu2017noisy,curtis2019exploiting}. 
    
    BB-like methods are very sensitive to noisy gradient estimates. Most existing stochastic BB algorithms \citep{NIPS2016_6286,LIANG2019197,8945980} overcome this issue with stabilization methods in the style of SVRG \citep{johnson2013accelerating}, which allows to prescribe a new step-size at every epoch only (i.e., after a full pass over the data). Doing so, one cannot capture variations of curvature within a full epoch, and one is limited to using absolute values to prevent negative step-sizes caused by non-convexity. On the contrary, our stochastic approximation method can adapt to local curvature every two iterations. Regarding the utilization of flatness and concavity of DL loss functions, the AdaBelief algorithm of \citet{zhuang2020adabelief} is worth mentioning. The latter uses the difference between the current stochastic gradient estimate and an average of past gradients, this difference being used to prescribe a vector step-size in the style of ADAM. In comparison, our method uses scalar step-sizes and aims to capture subtle variations as it computes a stabilized difference between consecutive gradient estimates \textit{before} averaging.
     
    There are few techniques to analyze stochastic methods in non-convex settings. An important category is the ODE machinery used for SGD \citep{davis2018stochastic,bolte2020}, ADAM \citep{barakat2018convergence} and INNA \citep{castera2019inertial}. In this paper, we use instead direct and more traditional arguments, such as in \citet{li2019convergence} in the context of DL.

    \section{Design of the algorithm}
    We first build a preliminary algorithm based upon a simple second-order variational model. We then adapt this algorithm to address mini-batch stochastic approximations.
    \begin{figure}[t]
     \centering
     \includegraphics[width=0.7\linewidth,trim={1.5cm 0cm 1.5cm  0cm},clip]{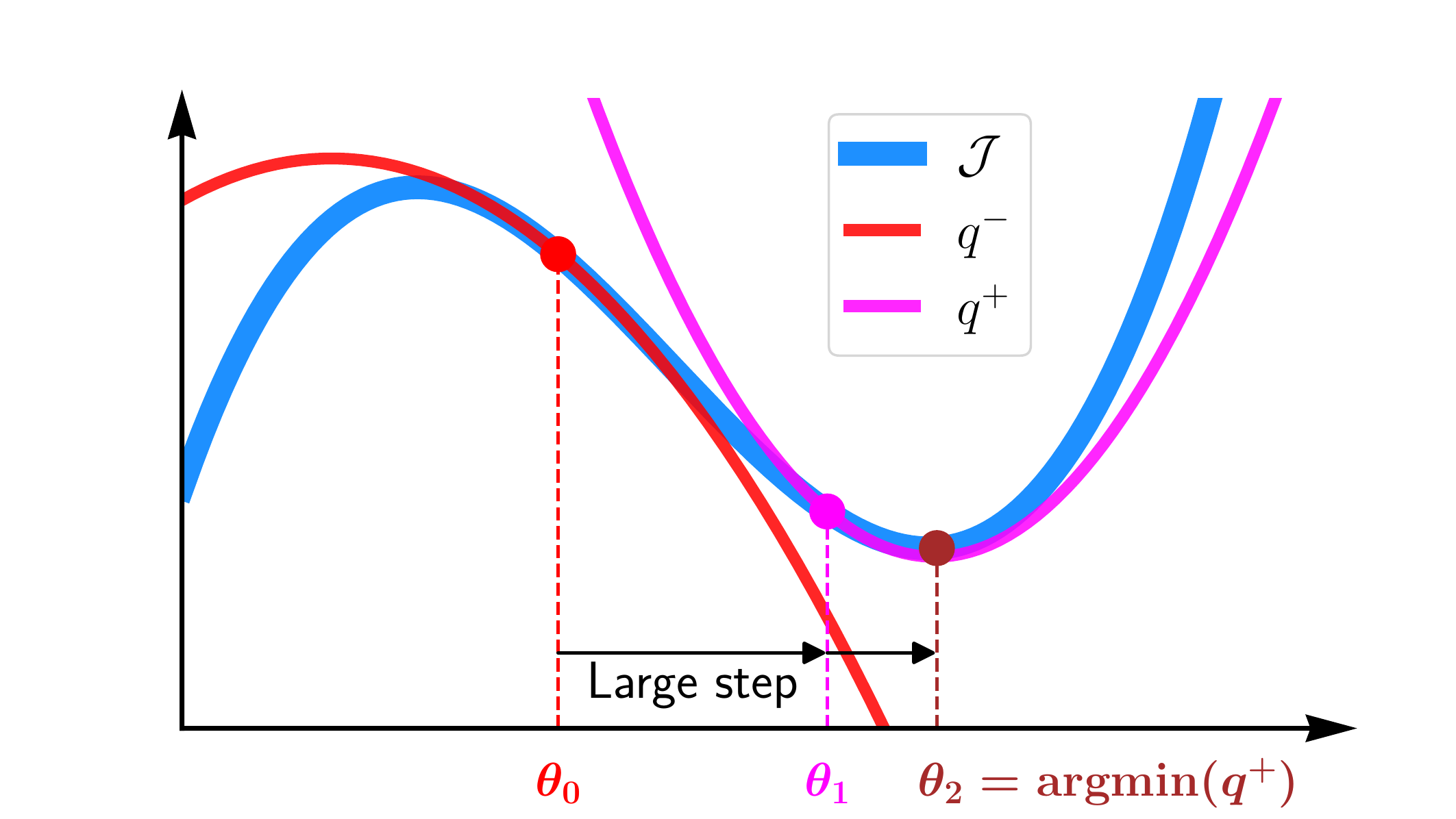}

        \caption{Illustration of negative and positive curvature steps. 
        The function $q^-$ represents the variational model at $\theta_0$, with negative curvature. Concavity suggests to take a large step to reach $\theta_1$.
        Then, at $\theta_1$, the variational model $q^+$ has positive curvature and can be minimized to obtain $\theta_2$. \label{fig::Illus}}
    \end{figure}
    \subsection{Deterministic full-batch algorithm}\label{sec::deterministic}
    We start with the following variational considerations.
   \subsubsection{Second-order infinitesimal step-size tuning.} Given a positive integer $P$, assume that $\J $ is a twice-differentiable function, denote $\nabla\J$ and $\nabla^2\J$ the gradient and the Hessian matrix of $\J$ respectively. Let $\theta\in\R^P$. Given an update direction $d\in\R^P$, a natural strategy is to choose $\gamma\in\R$ that minimizes $\J(\theta + \gamma d)$. Let us approximate $\gamma\mapsto\J(\theta+\gamma d)$ around $0$ with a Taylor expansion,
    \begin{equation}\label{eq::quadmodel}
        q_d(\gamma) \stackrel{\mathrm{def}}{=} \J(\theta) + \gamma\langle\nJ(\theta),d\rangle + \frac{\gamma^2}{2} \langle\nabla^2 \J(\theta)d,d\rangle.
    \end{equation} 
    If the curvature term $\langle\nabla^2 \J(\theta)d,d\rangle$ is positive, then $q_d$ has a unique minimizer at,
    \begin{equation}
        \gamma^\star=-\frac{\langle\nJ(\theta),d\rangle}{\langle\nabla^2 \J(\theta)d,d\rangle}.\label{eq::minofq}
    \end{equation}
    On the contrary when $\langle \nabla^2 \J(\theta)d,d\rangle\leq 0$, the infinitesimal model $q_d$ is concave (or equivalently $\J$ is locally concave in the direction $d$) which suggests to take a large step-size. These considerations are illustrated on Figure~\ref{fig::Illus}.
    
     \subsubsection{Tuning gradient descent.} In order to tune gradient  descent we choose the direction $d=-\nJ(\theta)$ which gives,
    \begin{equation}\label{eq::ACstep}
        \gamma(\theta) \stackrel{\mathrm{def}}{=} \frac{\Vert\nJ(\theta)\Vert^2}{\langle \nabla^2 \J(\theta)
        \nJ(\theta),\nJ(\theta)\rangle}.
    \end{equation} 
    According to our previous considerations, an ideal iterative process $\tkp = \tk -\gk\nabla\J(\tk)$
     would use $\gk=\gamma(\tk)$ when $\gamma(\tk)>0$. But for computational reasons and discretization purposes, we shall rather seek a step-size $\gamma_k$ such that, $\gk \simeq \gamma(\theta_{k-1})$ (again when $\gamma(\theta_{k-1})>0$). Let us assume that, for $k\geq1$,  $\theta_{k-1},\gamma_{k-1}$ are known and let us approximate the quantity,
      \begin{equation}
        \gamma(\theta_{k-1}) = \frac{\Vert\nJ(\theta_{k-1})\Vert^2}{\langle \nabla^2 \J(\theta_{k-1})
         \nJ(\theta_{k-1}),\nJ(\theta_{k-1})\rangle},
    \end{equation}
    using only first-order objects. 
    We rely on two identities,
     \begin{align}
            \Delta \tk &\stackrel{\mathrm{def}}{=} \tk - \tkm= -\gamma_{k-1}\nabla\J(\tkm),\\
            \Delta g_k &\stackrel{\mathrm{def}}{=} \nJ(\tk)-\nJ(\tkm)\simeq -\gamma_{k-1}\CJ(\tkm), \label{eq::deltaGk} 
    \end{align}
   where $\CJ(\theta) \stackrel{\mathrm{def}}{=} \nabla^2 \J(\theta)\nJ(\theta)$ and \eqref{eq::deltaGk} is obtained by Taylor's formula. Combining the above, we are led to consider the following step-size,
    \begin{equation}\label{eq::BB1}
        \gamma_k =\begin{cases}
         \frac{\Vert \dtk\Vert^2}{\langle\dtk,\dgk\rangle} \quad &\text{if $\langle\dtk,\dgk\rangle >0$}\\
        \quad\nu  &\text{otherwise}
        \end{cases},
    \end{equation}
    where $\nu>0$ is an hyper-parameter of the algorithm representing the large step-sizes to use in locally concave regions. 
    
    The resulting full-batch non-convex optimization method is Algorithm~\ref{alg::FB}, in which $\alpha$ is a so-called {\em learning rate} or {\em scaling factor}. 
    This algorithm is present in the literature under subtle variants  \citep{raydan1997barzilai,dai2002modified,XIAO20102986,Biglari2013ScalingOT}. It may be seen as a non-convex version of the BB method (originally designed for strongly convex functions) as it coincide with the BB step-size when $\langle\dtk,\dgk\rangle$ is positive.
    In the classical optimization literature the BB step-size and its variants are often combined with line-search procedures which is impossible in our large-scale DL context. This is why we replace the line-search by a scaling factor $\alpha$, present in most DL optimizers and which generally requires tuning. Our purpose is not however to get rid of hyper-parameters pre-tuning but rather to combine this $\alpha$ with an automatic fine tuning able to capture the local variations in $\J$. 
    Although Algorithm~\ref{alg::FB} is close to existing methods, the interest of our variational viewpoint is the characterization of the underlying geometrical mechanism supporting the algorithm, which is key in designing an efficient stochastic version of Algorithm~\ref{alg::FB} in Section~\ref{sec::mini-batch}.

    \begin{algorithm}[t]
        \caption{Full-batch preliminary algorithm
            }\label{alg::FB}
        \begin{algorithmic}[1]

            \STATE{\bfseries Input:} $\alpha>0$, $\nu>0$
            \STATE Initialize  $\theta_0\in\R^P$

            \STATE $\theta_1 = \theta_0-\alpha\nJ(\theta_0)$
            \FOR{$k=1,\ldots$}
            
            \STATE $\dgk = \nJ(\tk)- \nJ(\theta_{k-1})$
            \STATE $\dtk = \tk-\theta_{k-1}$
            \IF{$\langle\dgk,\dtk\rangle>0\textbf{}$}
            \STATE $\gamma_k = \frac{\Vert\dtk\Vert^2}{\langle\dgk,\dtk\rangle}$ 
            \ELSE
            \STATE $\gamma_k=\nu$
            \ENDIF
            \STATE $\theta_{k+1} = \tk - \alpha \gamma_k \nJ(\tk)$
            \ENDFOR
            
        \end{algorithmic}
    \end{algorithm}

    \begin{figure}[t]
            \centering
            \includegraphics[trim={0.2cm 0 0.2cm 0.2cm},clip ,width=0.65\linewidth]{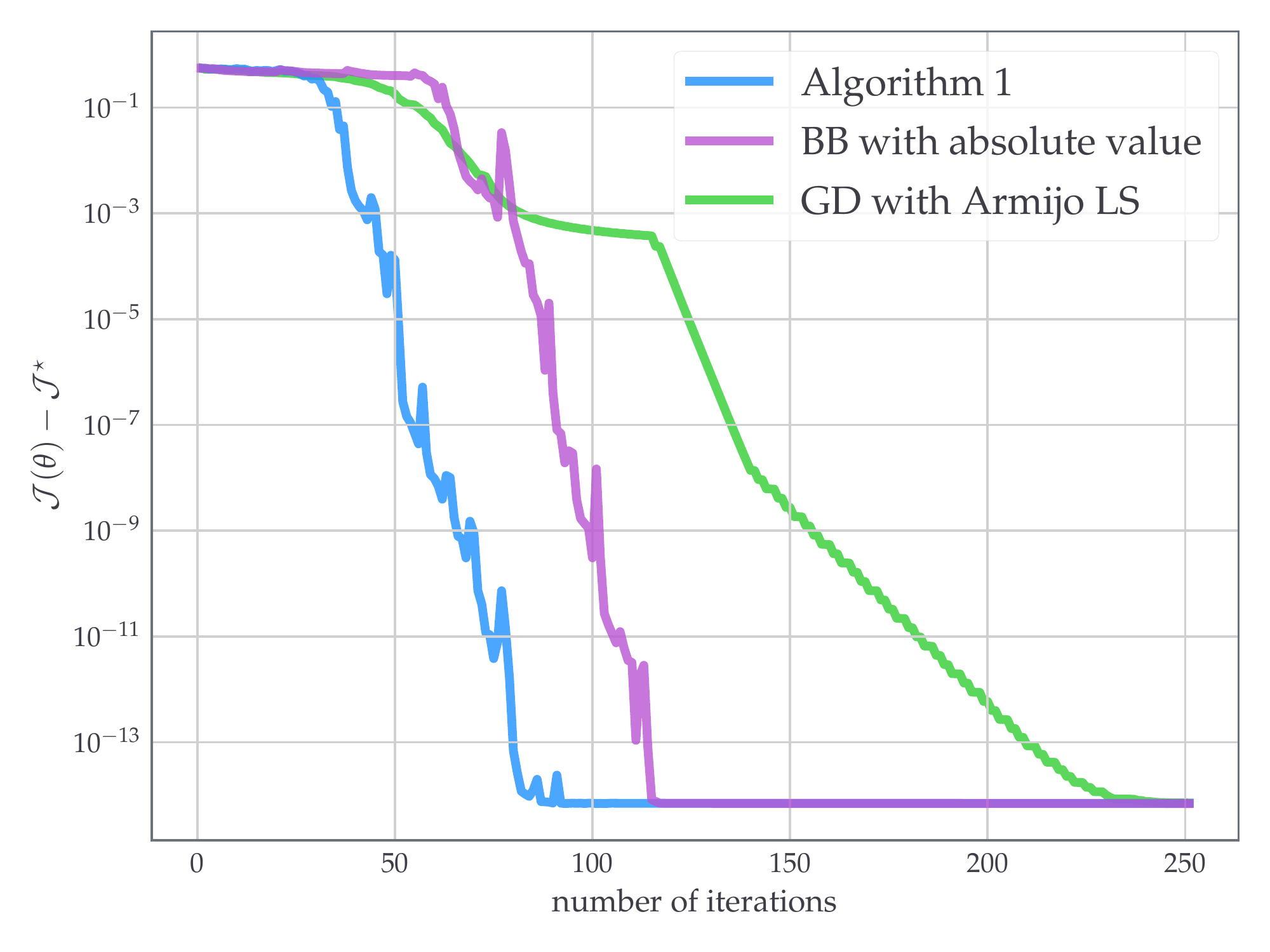}
        \caption{Values of the loss function $\J(\theta)$ against iterations (each corresponding to a gradient step) for the synthetic non-convex regression problem detailed in Section~\ref{sec::numCD} of the Supplementary. The optimal value $\J^\star$ is unknown and is estimated by taking the best value obtained among all algorithms after $10^{5}$ iterations. \label{fig::bbvsnc}}
    \end{figure}

        \subsubsection{Illustrative experiment.} Before presenting the stochastic version, we illustrate the interest of exploiting negative curvature through the \textit{large-step} parameter $\nu$ with a synthetic experiment inspired from \citet{carmon2017convex}. We apply Algorithm~\ref{alg::FB} to a non-convex regression problem of the form $\min_{\theta\in\R^P} \phi(A\theta - b)$ where $\phi$ is a non-convex real-valued function (see Section~\ref{sec::numCD} of the Supplementary). We compare Algorithm~\ref{alg::FB} with the methods \`a la BB where absolute values are used when the step-size is negative\footnote{For a fair comparison we implement this method with the scaling-factor $\alpha$ of Algorithm~\ref{alg::FB}.} (see, e.g.,  \citet{NIPS2016_6286,LIANG2019197} in stochastic settings) and with Armijo's line-search gradient method. As shown on Figure~\ref{fig::bbvsnc}, Algorithm~\ref{alg::FB} efficiently exploits local curvature and converges much faster than other 
        methods.

    \subsection{Stochastic mini-batch algorithm\label{sec::mini-batch}}
    We wish to adapt Algorithm~\ref{alg::FB} in cases where gradients can only be approximated through mini-batch sub-sampling. This is necessary in particular for DL applications.
    \subsubsection{Mini-batch sub-sampling.}
    We assume the following sum-structure of the loss function, for $N\in\N_{>0}$,
    \begin{equation}\label{eq::mainProblem}
        \J(\theta) = \frac{1}{N}\sum_{n=1}^{N} \J_n(\theta),
    \end{equation} 
    where each $\J_n$ is a twice continuously-differentiable function. Given any fixed subset $\mb \subset \{1,\ldots, N\}$, we define the following quantities for any $\theta \in \R^P$,
    \begin{align}\label{eq::approx}
        \J_{\mb}(\theta) &\stackrel{\mathrm{def}}{=} \frac{1}{\vert\mb\vert}\sum_{n\in\mb} \J_{n}(\theta),\\ \nJ_{\mb}(\theta) &\stackrel{\mathrm{def}}{=} \frac{1}{\vert\mb\vert}\sum_{n\in\mb} \nJ_{n}(\theta),
    \end{align}
    where $|\mb|$ denotes the number of elements of the set $\mb$.
    Throughout this paper we will consider independent copies of a random subset $\ms \subset \{1,\ldots, N\}$ referred to as mini-batch. The distribution of this subset is fixed throughout the paper and taken such that the expectation over the realization of $\ms$ in \eqref{eq::approx} corresponds to the empirical expectation in \eqref{eq::mainProblem}. This is valid for example if $\ms$ is taken uniformly at random over all possible subsets of fixed size. As a consequence, we have the identity $\J = \mathbb{E}[\J_{\ms}]$, where $\mathbb{E}$ denotes the expectation operator, here taken over the random draw of $\ms$. This allows to interpret mini-batch sub-sampling as a stochastic approximation process since we also have $\nJ = \mathbb{E}[\nJ_{\ms}]$.
    
     Our algorithm is of stochastic gradient type where stochasticity is related to the randomness of mini-batches. We start with an initialization $\theta_0 \in \R^P$, and a sequence of {i.i.d.} random mini-batches $(\mb_k)_{k\in\N}$, whose common distribution is the same as $\ms$. The algorithm produces a random sequence of iterates $(\theta_k)_{k\in \N}$. For $k\in \N$, $\mb_k$ is used to estimate an update direction $-\nJ_{\mb_k}(\theta_k)$ which is used in place of $-\nJ(\theta_k)$ in the same way as gradient descent algorithm.

    \begin{algorithm}[t]
        \caption{Step-Tuned SGD}\label{alg::Steptuned}
        \begin{algorithmic}[1]
            
            \STATE{\bfseries Input:} $\alpha>0$, $\nu>0$
            \STATE{\bfseries Input:} $\beta\in[0,1]$, $\tilde{m}>0$, $\tilde{M}>0$, $\delta\in(0
            ,1/2)$
            \STATE{\bfseries Initialize} $\theta_0\in\R^P$, $G_{-1} = \mathbf{0}_P$, $\gamma_0=1$
            \STATE{\bfseries Draw} independent random mini-batches $(\mb_k)_{k\in\N}$.
            \FOR{$k=0, 1,\ldots$}
            \STATE $\theta_{k+\frac{1}{2}} = \theta_{k} - \frac{\alpha}{(k+1)^{1/2+\delta}}\gamma_{k} \nJ_{\mb_k}(\theta_{k})$
            \STATE $\theta_{k+1}=\theta_{k+\frac{1}{2}} - \frac{\alpha}{(k+1)^{1/2+\delta}}\gamma_{k} \nJ_{\mb_k}(\theta_{k+\frac{1}{2}})$
            
            \STATE $\dtkh=\theta_{k+\frac{1}{2}}-\theta_{k}$
            \STATE $\dgkh=\nJ_{\mb_{k}}(\theta_{k + \frac{1}{2}})- \nJ_{\mb_{k}}(\theta_{k})$
            \STATE $G_{k}=\beta G_{k-1} + (1-\beta) \dgkh$
            \STATE $\hat{G}_{k}=G_{k}/(1-\beta^{k+1})$
            \IF{$\langle \hat{G}_{k},\dtkh\rangle>0$}
            \STATE $\gamma_{k+1}=\frac{\Vert\dtkh\Vert^2}{\langle \hat{G}_{k},\dtkh\rangle}$
            \ELSE
            \STATE $\gamma_{k+1}=\nu$
            \ENDIF 
            \STATE $\gamma_{k+1}=\min(\max(\gamma_{k+1},\tilde{m}),\tilde{M})$
            \ENDFOR
        \end{algorithmic}
    \end{algorithm}

    \subsubsection{Second-order tuning of mini-batch SGD: Step-Tuned SGD.}
    Our goal is to devise a step-size strategy, based on the variational ideas developed earlier and on the quantity $\CJ$, in the context of mini-batch sub-sampling. First observe that for $\theta\in\R^P$,
     \begin{equation}
        \CJ(\theta) = \nabla^2\J(\theta)\nabla\J(\theta) = \nabla\left(\frac{1}{2} \Vert\nabla\J(\theta)\Vert^2\right).
     \end{equation}  
     So rewriting $\J$ as an expectation,  
     \begin{equation}
        \CJ(\theta) = \nabla\left(\frac{1}{2} \Vert\mathbb{E}\left[\nabla\J_{\ms}(\theta)\right]\Vert^2\right),
     \end{equation}
    where $\ms$ denotes like in the previous paragraph a random subset of $\left\{1,\ldots,N\right\}$, or mini-batch. This suggests the following estimator, 
    \begin{equation} 
        \CJB(\theta) \stackrel{\mathrm{def}}{=} \nabla\left(\frac{1}{2} \Vert\nabla\J_{\mb}(\theta)\Vert^2\right) =\nabla^2\J_{\mb}(\theta)\nabla\J_{\mb}(\theta), \label{eq::CJBk}
    \end{equation}
   to build an infinitesimal model as in \eqref{eq::ACstep}, where  $\mb\subset\left\{1,\ldots,N\right\}$ and $\theta\in\R^P$.

    Like in the deterministic case we approximate the new target \eqref{eq::CJBk} with a Taylor expansion of $\J_{\mb}$ between two iterations. We obtain for any $\mb \subset \{1,\ldots,N\}$, $\theta\in\R^P$, and small $\gamma>0$
    \begin{equation} 
        -\gamma\CJB(\theta)\simeq \nJ_{\mb}(\underbrace{\theta - \gamma \nJ_{\mb}(\theta)}_{\text{next iterate}})  - \nJ_{\mb}(\theta).
        \label{eq::CJBkDiffGrad}
    \end{equation}
    This suggests to use each mini-batch twice and compute a difference of gradients every two iterations.\footnote{There is also the possibility of computing additional estimates as \citet{schraudolph2007stochastic} previously did for a stochastic BFGS algorithm, but this  would  double the  computational cost.} We adopt the following convention, at iteration $k\in\N$, the random mini-batch $\mb_k$ is used to compute a stochastic gradient, $\nJ_{\mb_{k}}(\theta_k)$ and at iteration $k+\frac{1}{2}$ the same mini-batch is used to compute another stochastic gradient $\nJ_{\mb_{k}}(\theta_{k+\frac{1}{2}})$, for a given $\theta_{k+\frac{1}{2}}$. Let us define,
    \begin{equation} 
        \dgkh \stackrel{\mathrm{def}}{=} \nJ_{\mb_{k}}(\theta_{k+\frac{1}{2}})- \nJ_{\mb_{k}}(\tk),\label{eq::MBdiscretization}
    \end{equation}
    thereby $\dgkh$ forms an approximation of $-\gamma_{k}\CJBk(\theta_k)$ that we can use to compute the next step-size $\gamma_{k+1}$. We define the difference between two iterates accordingly,
    \begin{equation} 
        \dtkh \stackrel{\mathrm{def}}{=} \theta_{k+\frac{1}{2}}- \tk.\label{eq::thetaMB}
    \end{equation}

    Finally, we stabilize the approximation of the target quantity in \eqref{eq::CJBk} by using an exponential moving average of the previously computed $(\Delta g_{\mb_j})_{j\leq k}$. More precisely, we recursively compute $G_k$ defined by,
    \begin{equation}\label{eq::movingAverage}
        G_{k} = \beta G_{k-1} + (1-\beta) \dgkh.
    \end{equation}
    We finally introduce $\hat{G}_{k}=G_{k}/(1-\beta^{k+1})$ to debias the estimator $G_k$ such that the sum of the weights in the average equals $1$. This mostly impacts the first iterations as $\beta^{k+1}$ vanishes quickly; a similar process is used in ADAM~\citep{kingma2014adam}.

    Altogether we obtain our main method: Algorithm~\ref{alg::Steptuned}, which we name \textit{Step-Tuned SGD}, as it aims to tune the step-size every two iterations and not at every epoch like most stochastic BB methods.
    Note that the main idea behind Step-Tuned SGD remains the same than in the deterministic setting: we exploit the curvature properties of $\J_{\mb_k}$ through the quantities $\langle\hat{G}_k,\dtkh\rangle$ to devise our method. Compared to Algorithm~\ref{alg::FB}, the iteration index is shifted by 1 so that the estimated step-size $\gamma_{k+1}$ only depends on mini-batches $\mb_0$ up to $\mb_k$ and is therefore conditionally independent of $\mb_{k+1}$. This conditional dependency structure is crucial to obtain the convergence guarantees given in Section~\ref{sec::theory}.
    \begin{figure}[t]
        \centering
        \centering
        \includegraphics[width=0.65\linewidth]{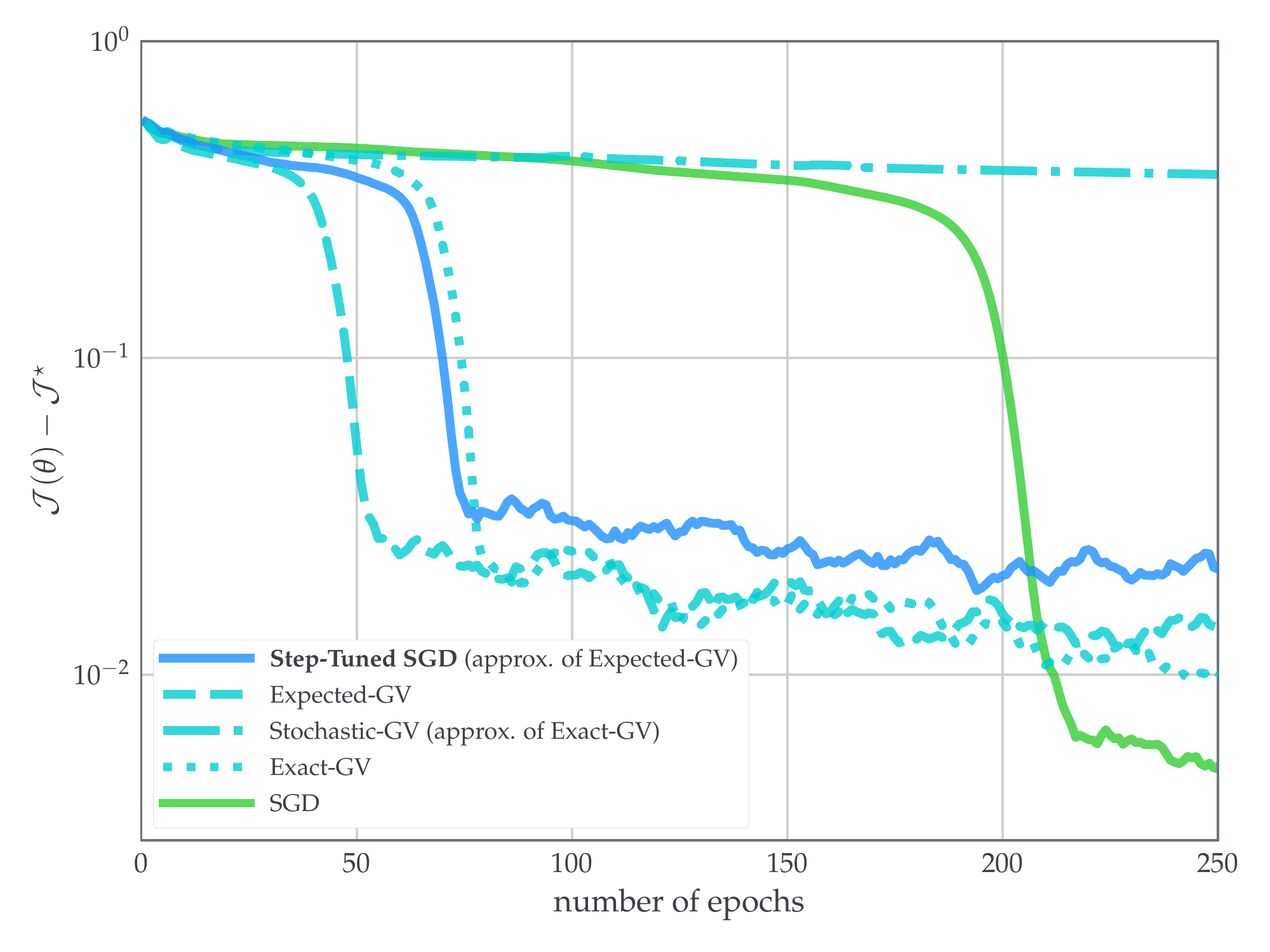}
        \caption{Values of the loss function against epochs for non-convex regression: heuristic methods (dashed lines) of Section~\ref{sec::heuristic} are compared with Step-Tuned SGD (plain blue). SGD serves as a reference to evidence the fast drop down effect of other methods. The additional computational cost of Expected-GV and Exact-GV is ignored as these methods are here only for illustration purposes (see Section~\ref{sec::heuristic}).\label{fig::CD_MB}}
    \end{figure}
    \begin{remark}
        It is worth precising that like most methods, Algorithm~\ref{alg::Steptuned} does not alleviate the need of tuning the scaling factor $\alpha$. The choice of this parameter remains indeed important in most practical applications. The purpose of Step-Tuned SGD is rather to speed up SGD by \textit{fine-tuning} the step-size at each iteration (through the introduction of $\gamma_k$). This is analogous to the BB methods for deterministic applications which often accelerate algorithms but must be stabilized with line-search strategies. To the best of our knowledge, in comparison to the \textit{epoch-wise} BB methods \citep{NIPS2016_6286,LIANG2019197}, Algorithm~\ref{alg::Steptuned} is the first method that manages to mimic the \textit{iteration-wise} behavior of deterministic BB methods for mini-batch applications.
    \end{remark}
    
    \subsection{Heuristic construction of Step-Tuned SGD}\label{sec::heuristic}
    In this section we present the main elements which led us to the step-tuned method of Algorithm \ref{alg::Steptuned} and discuss its hyper-parameters. Throughout this paragraph, the term \textit{gradient variation} (GV)  denotes the local variations of the gradient; it is simply the difference of consecutive gradients along a sequence. Our heuristic discussion blends discretization arguments and experimental considerations. We use the  non-convex regression experiment of Section~\ref{sec::deterministic} as a test for our intuition and algorithms. A complete description of the methods below is given in Section~\ref{supp::otherAlgs}, we only sketch the main ideas.
    \subsubsection{First heuristic experiment with exact GVs.} Assume that along any ordered collection $\theta_1,\ldots,\theta_k \in \R^P$, one is able to evaluate the GVs of $\J$, that is, terms of the form $\nabla \J(\theta_{i}) - \nabla \J(\theta_{i-1})$. Recall that we denote $\Delta\theta_i = \theta_i-\theta_{i-1}$, the difference between two consecutive iterates, for all $i \geq 1$. In the deterministic (i.e., noiseless) setting, Algorithm \ref{alg::FB} is based on these GVs, indeed,
    \begin{align}
        \label{eq::algoDeterministic}
        \theta_{k+1} = \theta_k -  \frac{ \alpha\|\dtk\|^2}{\left\langle \nabla \J(\theta_k) - \nabla \J(\theta_{k-1}), \dtk\right\rangle} \nabla \J(\theta_k),
    \end{align}
    whenever the denominator is positive. Given our sequence of independent random mini-batches $(\mb_k)_{k\in \N}$, a heuristic stochastic approximation version of this recursion could be as follows,
    \begin{align}
    \label{eq::algoExpe1}
        \theta_{k+1} = \theta_k -  \frac{ \alpha_k\|\dtk\|^2}{\left\langle \nabla \J(\theta_k) - \nabla \J(\theta_{k-1}), \dtk\right\rangle} \nabla \J_{\mb_k}(\theta_k),
        \tag{Exact-GV}
    \end{align}
    where the difference between \eqref{eq::algoDeterministic} and \ref{eq::algoExpe1} lies in the randomness of the search direction and the dependency of the scaling factor $\alpha_k$ which aims to moderate the effect of noise (generally $\alpha_k\to 0$). As shown in Figure~\ref{fig::CD_MB} the recursion \ref{eq::algoExpe1} is much faster than SGD especially for the first $\sim 150$ epochs which is often the main concern for DL applications. Indeed, although SGD achieves a smaller value of $\J$ after a larger number of iterations (due to other methods using larger step-sizes), this happens when the value of $\J$ is already low.\footnote{Step-Tuned SGD achieves the same small level of error as SGD when doing additional epochs thanks to the decay schedule present in Algorithm~\ref{alg::Steptuned}.} Overall the quantity \ref{eq::algoExpe1} seems very promising.
    
    Yet, for large sums, the gradient-variation in \ref{eq::algoExpe1} is too computationally expensive. One should therefore adapt \eqref{eq::algoExpe1} to the mini-batch context. A direct adaption would simply consist in the algorithm,
    \begin{align}
    \label{eq::algoExpe1Sto}
        \theta_{k+1} =  \theta_k -  \frac{ \alpha_k\|\dtk\|^2}{\left\langle \nabla \J_{\mb_k}(\theta_k) - \nabla \J_{\mb_{k-1}}(\theta_{k-1}), \dtk\right\rangle} \nabla \J_{\mb_k}(\theta_k),
        \tag{Stochastic-GV}
    \end{align}
    where mini-batches are used both to obtain a search direction and to approximate the GV.
    For this naive approach, we observe a dramatic loss of performance, as illustrated in Figure~\ref{fig::CD_MB}. This reveals the necessity to use accurate stochastic ap\-proxi\-mation of GVs.
    
    \subsubsection{Second heuristic experiment using expected gradient variations.} Towards a more stable approximation of the GVs, we consider the following recursion,
    \begin{align}
    \label{eq::algoExpe2}
        \theta_{k+1} = \theta_k -  \frac{ \alpha_k\|\dtk\|\| \nabla\J_{\mb_{k-1}}(\theta_{k-1}) \|}{\left\langle -\mathbb{E}[\mathcal{C}_{\mathcal{J}_{\ms}}(\theta_{k-1})], \dtk \right\rangle} \nabla \J_{\mb_k}(\theta_k),
        \tag{Expected-GV}
    \end{align}
    where $\mathcal{C}_{\mathcal{J}_{\mb}}$ is defined in \eqref{eq::CJBk} for any $\mb \subset \{1,\ldots,N\}$ and the expectation taken is over the independent draw of $\ms \subset \{1,\ldots,N\}$, conditioned on the other random variables. The main difference with \ref{eq::algoExpe1} is the use of expected GVs instead of  exact GVs, the minus sign ensures a coherent interpretation in term of GVs. The numerator in \ref{eq::algoExpe2} is also modified to ensure homogeneity of the steps with the other variations of the algorithm. Indeed $\mathcal{C}_{\mathcal{J}_{\mb}}(\theta_k)$ approximates a difference of gradients modulo a step-size, see \eqref{eq::CJBkDiffGrad}. As illustrated in Figure~\ref{fig::CD_MB}, the recursion \ref{eq::algoExpe2} provides performances comparable (and even superior) to \ref{eq::algoExpe1}, and in particular for both algorithms, we also recover the loss drop which was observed in the deterministic setting.

Algorithm~\ref{alg::Steptuned} is nothing less than an approximate version of \ref{eq::algoExpe2} which combines a double use of mini-batches with a moving average. Indeed, from \eqref{eq::CJBkDiffGrad}, considering the expectation over the random draw of $\ms$, for any $\theta \in \R^P$ and small $\gamma > 0$, we have,
\begin{align}
    -\gamma \mathbb{E}[\mathcal{C}_{\mathcal{J}_{\ms}}(\theta)] \simeq\; & \mathbb{E}\left[\nJ_\ms\left(\theta - \gamma \nJ_{\ms}(\theta)\right)  - \nJ_{\ms}(\theta)\right].
    \label{eq::approxCJdiffGrad}
\end{align}
 The purpose of the term $\hat{G}_{k}$ in Algorithm~\ref{alg::Steptuned} is precisely to mimic this last quantity. The experimental results of Algorithm~\ref{alg::Steptuned} are very similar to those of \ref{eq::algoExpe2}, see Figure~\ref{fig::CD_MB}. 
 
Let us conclude by saying that the above considerations on gradient variations (GVs) led us to propose Algorithm~\ref{alg::Steptuned} as a possible stochastic version of Algorithm~\ref{alg::FB}. The similarity between the performances of the two methods and the underlying geometric aspects (see Section~\ref{sec::mini-batch}) were also major motivations.

    \subsubsection{Parameters of the algorithm.}
    Algorithm~\ref{alg::Steptuned} contains more hyper-parameters than in the deterministic case,  but we recommend keeping the default values for most of them.\footnote{Default values: $(\nu,\beta,\tilde{m},\tilde{M},\delta) = (2,0.9,0.5,2,0.001)$} Like in most optimizers (SGD, ADAM, RMSprop, etc.), only the parameter $\alpha>0$ has to be carefully tuned to get the most of Algorithm~\ref{alg::Steptuned}. The value $\beta=0.9$ is a common default choice for exponential moving averages (see e.g., \citet{kingma2014adam}). Note that we enforce $\gamma_k \in [\tilde{m}$, $\tilde{M}]$. The bounds stabilize the algorithm and also play an important role for the convergence as we will show in Section~\ref{sec::theory}. While a fine tuning of these bounds may improve the performances, we chose rather tight default values for the sake of numerical stability so that practitioners need not tuning them. The same choice was made for the parameter $\nu$.  Note that we also enforce the step-size to decrease using a decay of the form $1/k^{1/2+\delta}$ where the value of $\delta>0$ is of little importance as long as it is taken close to $0$. This standard procedure goes back to \citet{robbins1951stochastic} and is again necessary to obtain the convergence results presented next.

    \section{Theoretical results\label{sec::theory}}
    We study the convergence of Step-Tuned SGD for smooth non-convex stochastic optimization which encompasses in particular smooth DL problems.
    \subsection{Main result.}
    We recall that $\J$ is a finite sum of twice continuously-differentiable functions $(\J_n)_{n=1,\ldots, N}$. Hence, the gradient of $\J$ and the gradients of each $\J_n$ are locally Lipschitz continuous. A function $g$ is locally Lipschitz continuous on $\R^P$ if for any $\theta\in\R^P$, there exists a neighborhood $\mathrm{V}$ of $\theta$ and a constant $L\in \R_+$ such that for all $\psi_1,\psi_2\in\mathrm{V}$,
    \begin{equation}
        \Vert g(\psi_1) -g(\psi_2) \Vert \leq  L\Vert \psi_1-\psi_2\Vert.
    \end{equation}
    We assume that $\J$ is lower-bounded on $\R^P$, which holds for most DL loss functions by construction (they are usually non-negative).
    We denote by $\Nhalf=\{0,\frac{1}{2},1,\frac{3}{2},2,\ldots\}$ the set of half integers so that the iterations of  Step-Tuned SGD are indexed by $k\in\Nhalf$.
    The main theoretical result of this paper follows. 
    \begin{theorem}\label{thm::mainres}
        Let $\theta_0\in\R^P$, and let $(\theta_k)_{k\in\Nhalf}$ be a sequence generated by Step-Tuned SGD initialized at $\theta_0$. Assume that there exists a constant $C_1>0$ such that almost surely $\sup_{k\in\Nhalf} \Vert \theta_k\Vert<C_1$. Then the sequence of values $(\J(\theta_k))_{k\in\N}$ converges almost surely and $\left(\Vert\nJ(\theta_k)\Vert^2\right)_{k\in\N}$ converges to $0$ almost surely. In addition, for $k\in\N_{>0}$,
        $$\displaystyle\min_{j\in\{0,\ldots,k-1\}}\esp{\Vert\nJ(\theta_j)\Vert^2} = O\left( \frac{1}{k^{1/2-\delta}}\right).$$
    \end{theorem}
    
        The results above state in particular that a realization of the algorithm reaches a point where the gradient is arbitrarily small with probability one. Note that the rate depends on the parameter $\delta\in(0,1/2)$ which can be chosen by the user and corresponds to the decay schedule $1/(k+1)^{1/2+\delta}$. In most cases, one will want to slowly decay the step-size so $\delta\simeq 0$ and the rate is close to $1/\sqrt{k+1}$.
        
        \subsection{An alternative to the boundedness assumption.}
        In Theorem~\ref{thm::mainres} the assumption that almost surely the iterates $(\theta_k)_{k\in\Nhalf}$ are uniformly bounded is made. While this is usual for non-convex problems tackled with stochastic algorithms \citep{davis2018stochastic,duchi2018stochastic,castera2019inertial} it may be hard to check in practice. This assumption is however consistent with numerical experiments since practitioners usually choose the hyper-parameters (in particular the step-size) so that the weights of the DNN remain ``not too large'' for the sake of numerical stability.
        
        One can alternatively replace the boundedness assumption by leveraging additional regularity assumptions on the loss function as \citet{li2019convergence} did for example for the scalar variant of ADAGRAD. This is more restrictive than the locally-Lipschitz-continuous property of the gradient that we used but for completeness we provide below an alternative version of Theorem~\ref{thm::mainres} where we assume that for each $n\in\{1,\ldots,N\}$, the function $\J_n$ and its gradient $\nabla\J_n$ are Lipschitz continuous.

\begin{corollary}\label{cor::globLipschitz}
    Let $\theta_0\in\R^P$, and let $(\theta_k)_{k\in\Nhalf}$ be a sequence generated by Step-Tuned SGD initialized at $\theta_0$. Assume that each $\J_n$ and $\nJ_n$ are Lipschitz continuous on $\R^P$ and that each $\J_n$ is bounded below, for all $n\in\{1,\ldots, N\}$. Then the same conclusions as in Theorem~\ref{thm::mainres} apply.
\end{corollary}
Note that the assumption that each $\J_n$ is Lipschitz continuous is similar to assumptions (H4) and (H4') from \citet[Section~3]{li2019convergence}, yet a little less general. This is due to the fact that using each mini-batch twice brings additional difficulties when studying the convergence (see the proof of Theorem~\ref{thm::mainres}).

\subsection{Proof sketch of Theorem~\ref{thm::mainres}.} The proof of our main theorem is fully detailed in Section~\ref{sec::supTheor} of the Supplementary. Here we present the key elements of this proof. 
\begin{itemize}
    \item The proof relies on the descent lemma: for any compact subset  $\mathsf{C}\subset\R^P$ there exists $L>0$ such that for any $\theta\in\mathsf{C}$ and  $d\in \R^P$ such that $\theta+d\in\mathsf{C}$,
    \begin{equation}\label{eq::genDescentSKETCH}
        \J(\theta+d) \leq \J(\theta) + \langle \nJ(\theta), d\rangle + \frac{L}{2}\Vert d \Vert^2.
    \end{equation}
    \item Let $(\theta_k)_{k\in\Nhalf}$ be a realization of the algorithm. Using the boundedness assumption, almost surely the iterates belong to a compact subset $\mathsf{C}$ on which $\nJ$ and the gradients estimates $\nJ_{\mb_{k}}$ are uniformly bounded. So at any iteration $k\in\N$, we may use the descent lemma \eqref{eq::genDescentSKETCH} on the update direction $d=-\gamma_k\nJ_{\mb_k}(\theta_k)$ to bound the difference $\J(\theta_{k+1})-\J(\theta_k)$.
    \item As stated in Section~\ref{sec::mini-batch}, conditioning on $\mb_0,\ldots,\mb_{k-1}$ the step-size $\gamma_k$ is constructed to be independent of the current mini-batch $\mb_k$. Using this and the descent lemma, we show that there exist $M_1, M_2>0$ such that, for all $k \in \N_{>0}$,
    \begin{equation}\label{eq::esptotSKETCH}
        \esp{\J(\theta_{k+1})\mid \mb_0,\ldots\mb_{k-1}} \leq\J(\theta_k)  - \frac{M_1}{(k+1)^{1/2+\delta}}\Vert \nJ(\theta_k)\Vert^2 + \frac{M_2}{(k+1)^{1+2\delta}},
    \end{equation}
    where $\esp{\J(\theta_{k+1})\mid \mb_0,\ldots\mb_{k-1}}$ denotes the conditional expectation of the random value $\J(\theta_{k+1})$ conditionally on the mini-batches $\mb_0,\ldots\mb_{k-1}$. 
    \item Applying Robbins-Siegmund  convergence theorem \cite{robbins1971convergence} for martingales to \eqref{eq::esptotSKETCH}, using the fact that $\sum_{k=0}^{+\infty}\frac{1}{(k+1)^{1+2\delta}}<\infty$, we obtain almost surely that the sequence $(\J(\tk))_{k\in\N}$ converges and
    \begin{equation}\label{eq::cvsumSKETCH}
        \sum_{k=0}^{+\infty}\frac{1}{(k+1)^{1/2+\delta}} \Vert \nJ(\theta_k) \Vert^2<+\infty,
    \end{equation}
    Since $\sum_{k=0}^{+\infty}\frac{1}{(k+1)^{1/2+\delta}}=+\infty$, we deduce that $ \nJ(\theta_k)$ converges to zero almost surely, using the local Lipschitz continuity of the gradient (from twice differentiability) and an argument of \citet{alber1998projected}. The rate follows from considering expectations on both sides of \eqref{eq::esptotSKETCH}.
\end{itemize}

    \section{Application to Deep Learning\label{sec::numDL}}
    We finally evaluate the performance of Step-Tuned SGD by training DNNs. We consider six different problems presented next and fully-specified in Section~\ref{sec:expe} of the Supplementary. The results with Problems (a) to (d) are first presented here while the results with Problems (e) and (f) are discussed at the end of this section. We compare Step-Tuned SGD with two of the most popular non-momentum methods, SGD and RMSprop~\cite{tieleman2012lecture}, and we also consider the momentum method ADAM~\citep{kingma2014adam} which is a very popular DL optimizer. 
    Our method is detailed below.
    \subsection{Setting of the experiments}\label{sec::settingDL}
    \begin{itemize}
        \item We consider image classification problems with color images CIFAR-10 and CIFAR-100 \citep{krizhevsky2009learning} and the training of an auto-encoder on MNIST \citep{lecun2010mnist}.
        \item The networks are slightly modified versions of Lenet \citep{lecun1998gradient}, ResNet-20 \citep{he2015}, Network-in-Network (NiN) \citep{lin2013network} and the auto-encoder of \citet{hinton2006reducing}.
        \item As specified in Table~\ref{tab::expsummary} of the Supplementary, we used either smooth (ELU, SiLU) or nonsmooth (ReLU) activations. Though our theoretical analysis only applies to smooth activations, we did not in practice observe a significant qualitative difference between  ReLU or its smooth versions.
        
        \item For image classification tasks, the dissimilarity measure is the cross-entropy, and for the auto-encoder, it is the mean-squared error. In each problem we also add a $\ell^2$-regularization parameter (a.k.a. weight decay) of the form $\frac{\lambda}{2}\Vert\theta\Vert_{2}^2$.
        \item For each algorithm, we selected the learning rate parameter $\alpha$ from the set $\{10^{-4},\ldots,10^0\}$. The value is selected as the one yielding minimum training loss after $10\%$ of the total number of epochs. For example, if we intend to train the network during $100$ epochs, the grid-search is carried on the first $10$ epochs. For Step-Tuned SGD, the parameter $\nu$ was selected with the same criterion from the set $\{1,2,5\}$ and for ADAM the momentum parameter was chosen in $\{0.1,0.5,0.9,0.99\}$. All other parameters of the algorithms are left to their default values.
        \item  Decay-schedule: To meet the conditions of Theorem~\ref{thm::mainres} the step-size decay schedule of SGD and Step-Tuned SGD takes the form $1/q^{1/2+\delta}$ where $q$ is the current epoch index and $\delta=0.001$. It slightly differs from what is given in Algorithm~\ref{alg::Steptuned} as we apply the decay at each epoch instead of each iteration. This slower schedule still satisfies the conditions for the convergence of Theorem~\ref{thm::mainres}. \footnote{An alternative common practice consists in manually decaying the step-size at pre-defined epochs. This technique although efficient in practice to achieve state-of-the-art results makes the comparison of algorithms harder, hence we stick to a usual Robbins-Monro type of decay.}
        RMSprop and ADAM rely on their own adaptive procedure and are usually used without step-size decay schedule.
        \item The experiments were run on a Nvidia GTX 1080 TI GPU, with an Intel Xeon 2640 V4 CPU. The code was written in \textit{Python} 3.6.9 and \textit{PyTorch} 1.4 \citep{paszke2019pytorch}.
    \end{itemize}
    
    \paragraph{Second experiment: mini-batch sub-sampling.}
    Step-Tuned SGD departs from the standard process of drawing a new mini-batch after each gradient update. Indeed, we use each mini-batch twice in order to properly approximate the curvature of the sliding loss, but also to maintain a computing time similar to standard algorithms.
    We performed additional experiments to understand the consequences of using the same mini-batch twice, and in particular make sure that this is not the source of the observed advantage of Step-Tuned SGD. In these experiments all methods are used with the mini-batch drawing procedure of Step-Tuned SGD detailed in Algorithm~\ref{alg::Steptuned} (each mini-batch being used to perform two consecutive gradient steps).
    
    \subsection{Results}
    \begin{figure}[t]
        \begin{center}
            \begin{minipage}[c]{0.25\textwidth}
                \centering
                {\centering \scriptsize Problem (a): training error}
                \includegraphics[width=\linewidth]{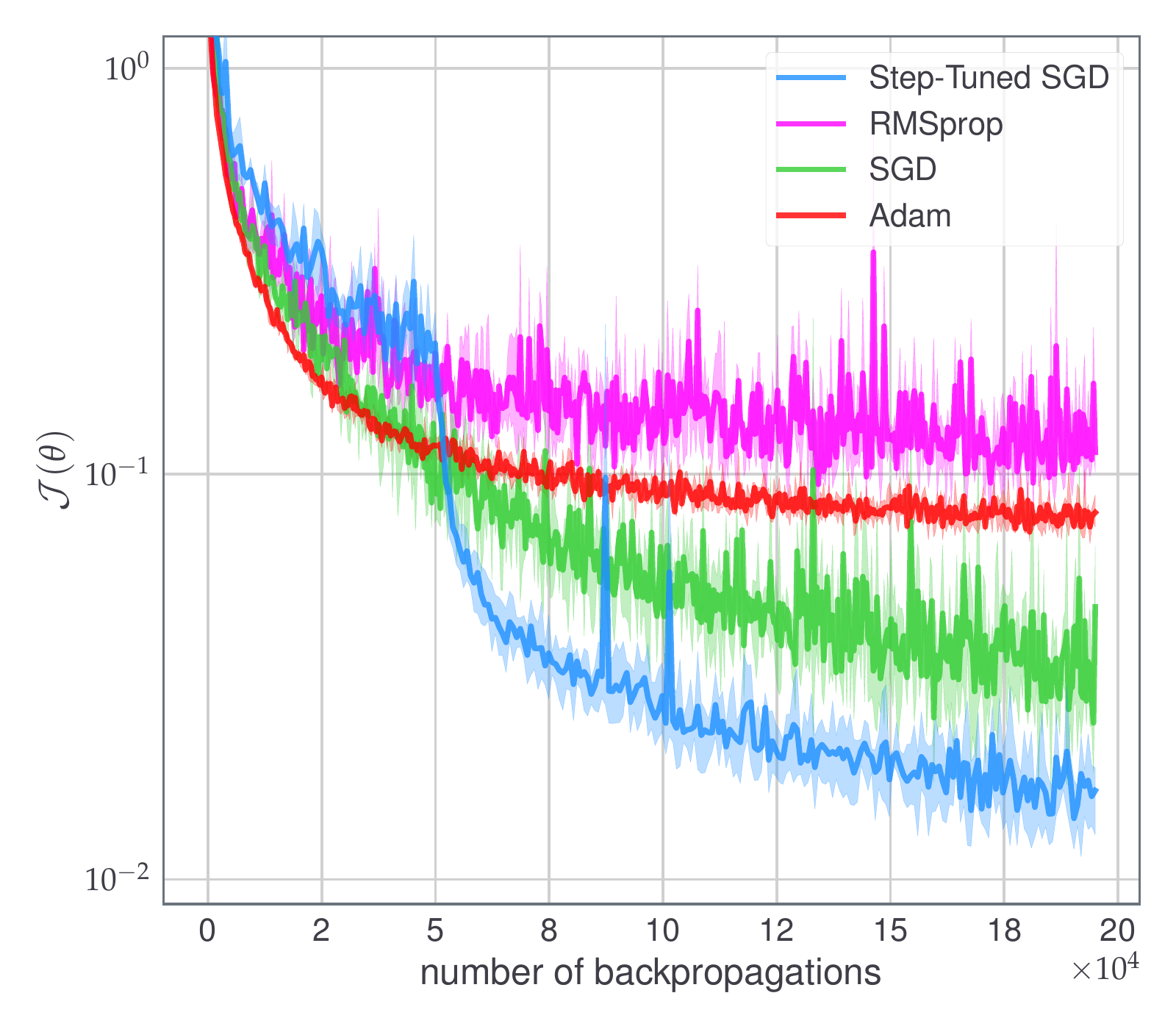}
            \end{minipage}%
            \begin{minipage}[c]{0.25\textwidth}
            \centering
            {\centering \scriptsize Problem (b): training error}
            \includegraphics[ width=\linewidth]{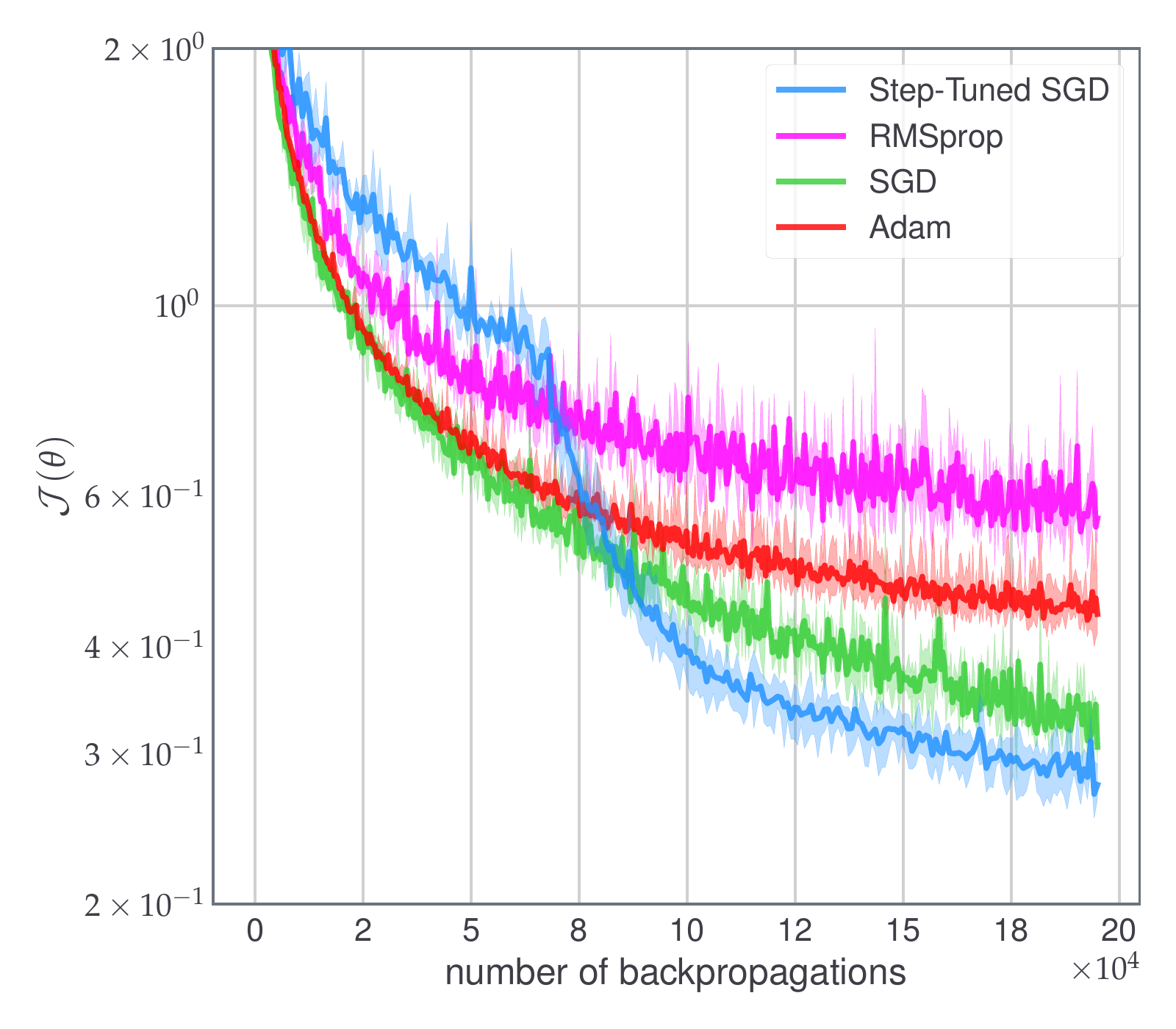}
            \end{minipage}%
            \begin{minipage}[c]{0.25\textwidth}
                \centering
                {\centering \scriptsize Problem (c): training error}
                \includegraphics[ width=\linewidth]{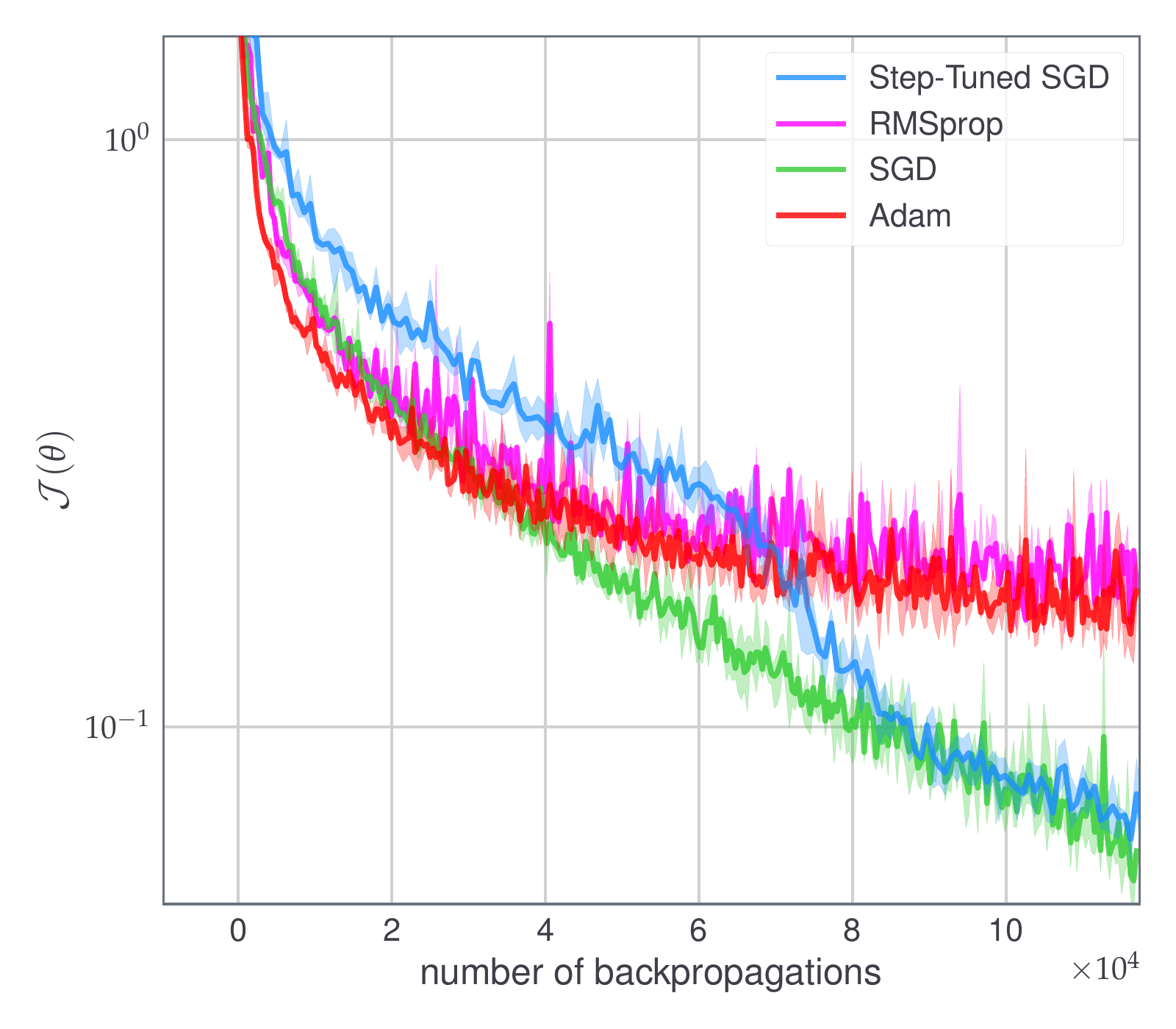}
            \end{minipage}%
            \begin{minipage}[c]{0.25\textwidth}
            \centering
            {\centering \scriptsize Problem (d): training error}
            \includegraphics[width=\linewidth]{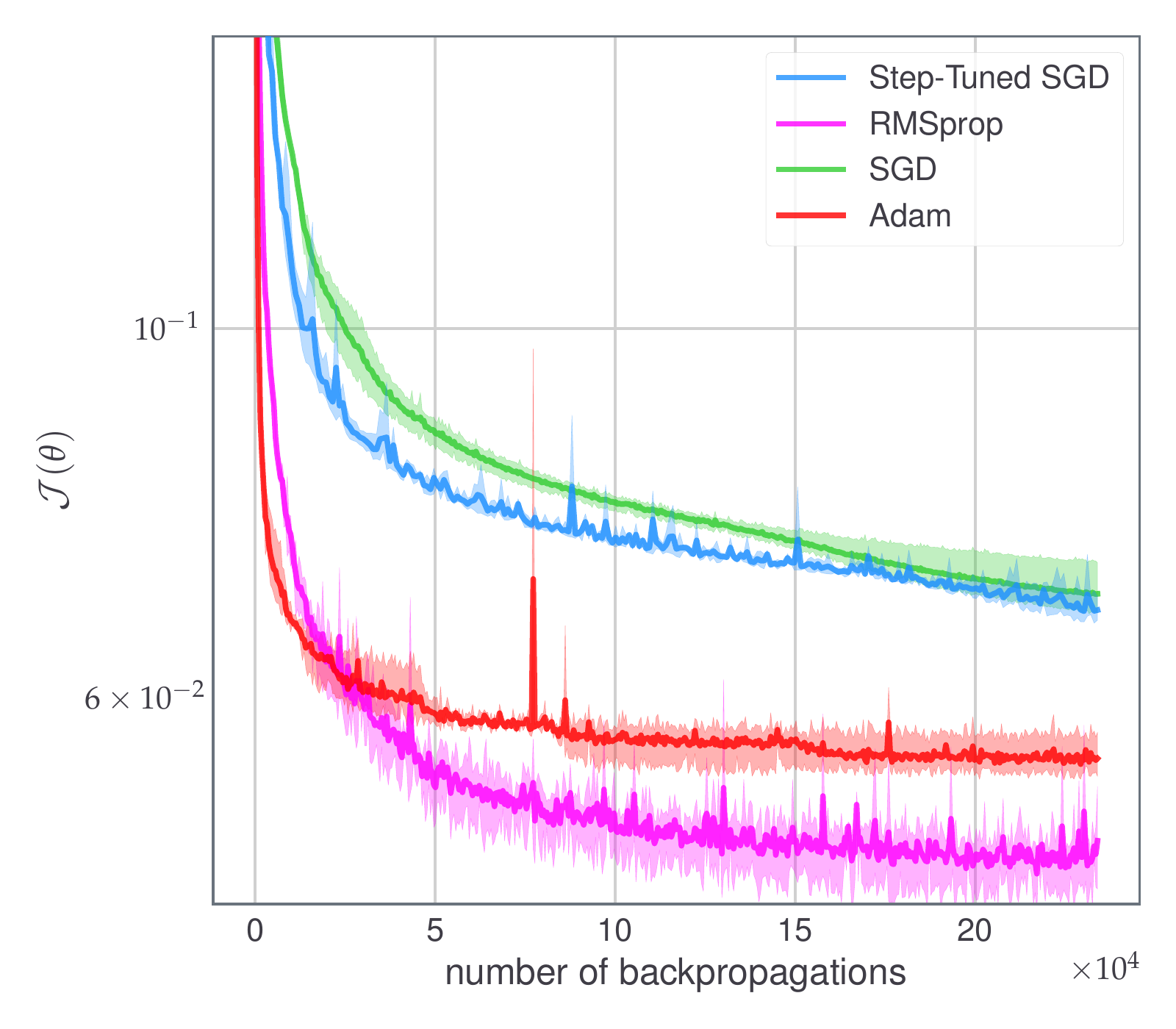}
            \end{minipage}%
            \\
            \begin{minipage}[c]{0.25\textwidth}
                \centering
                {\centering \scriptsize Problem (a): test accuracy}
                \includegraphics[ width=\textwidth]{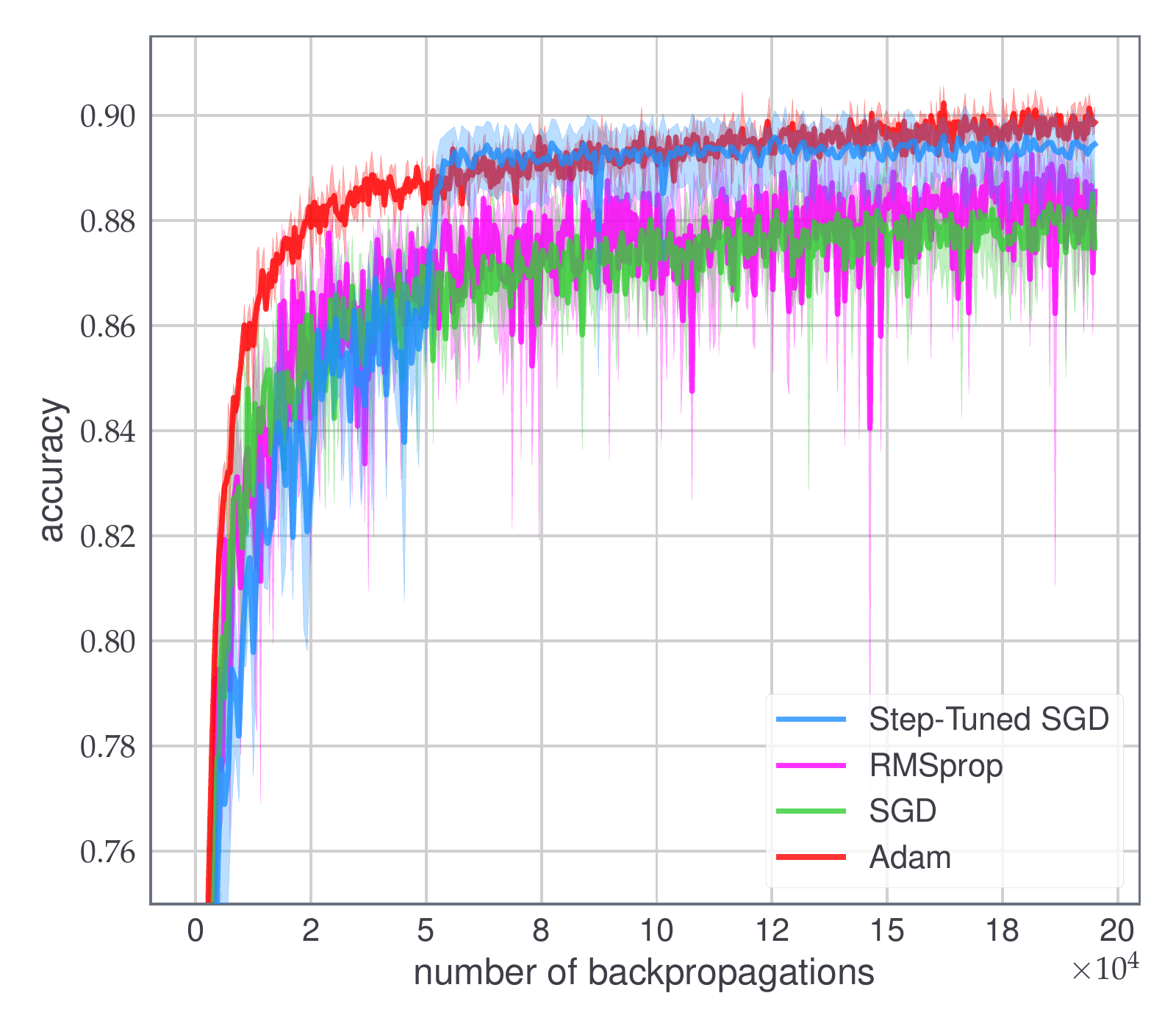}
            \end{minipage}%
            \begin{minipage}[c]{0.25\textwidth}
            \centering
            {\centering \scriptsize Problem (b): test accuracy}
            \includegraphics[ width=\textwidth]{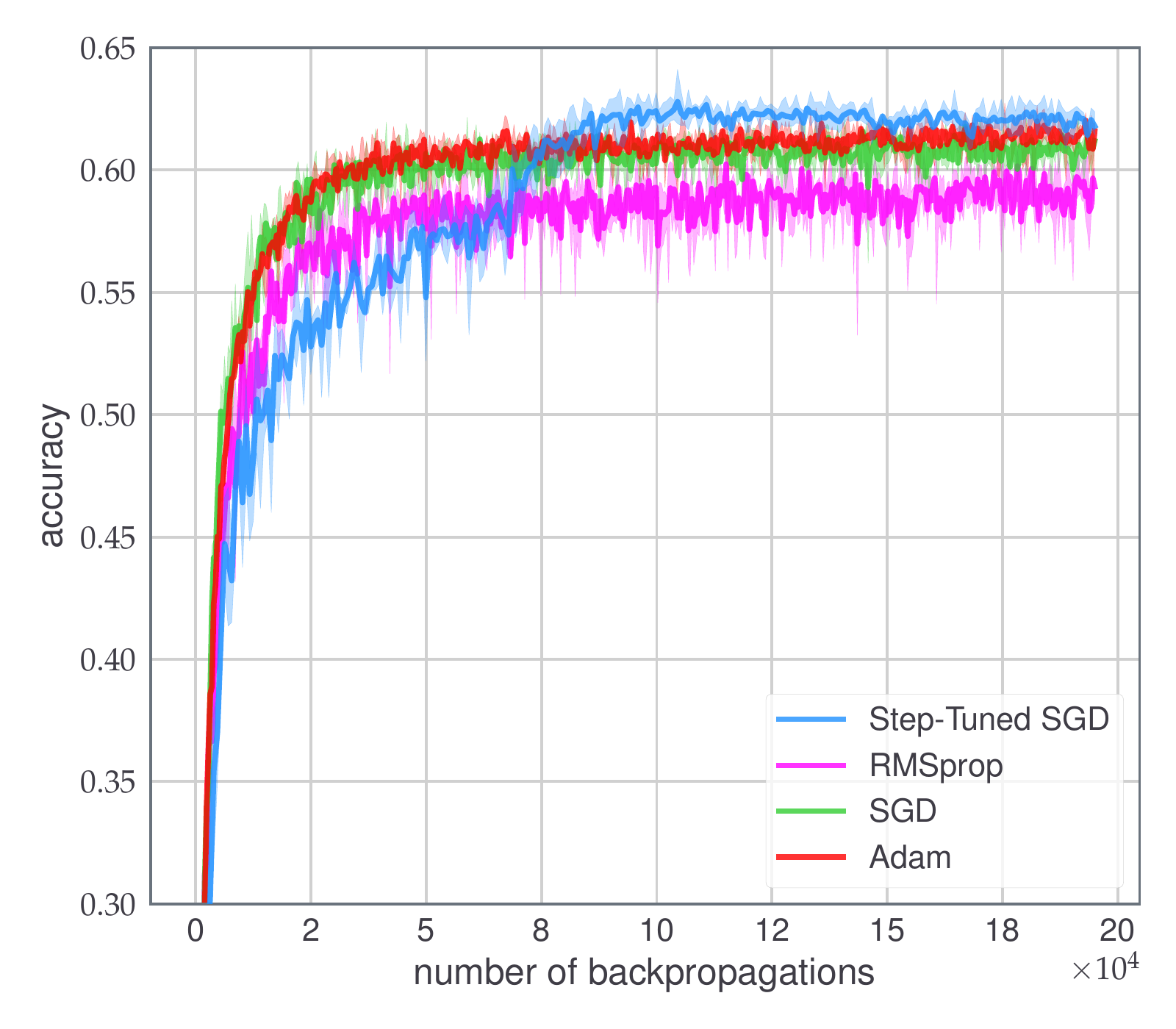}
            \end{minipage}%
            \begin{minipage}[c]{0.25\textwidth}
                \centering
                {\centering \scriptsize Problem (c): test accuracy}
                \includegraphics[ width=\linewidth]{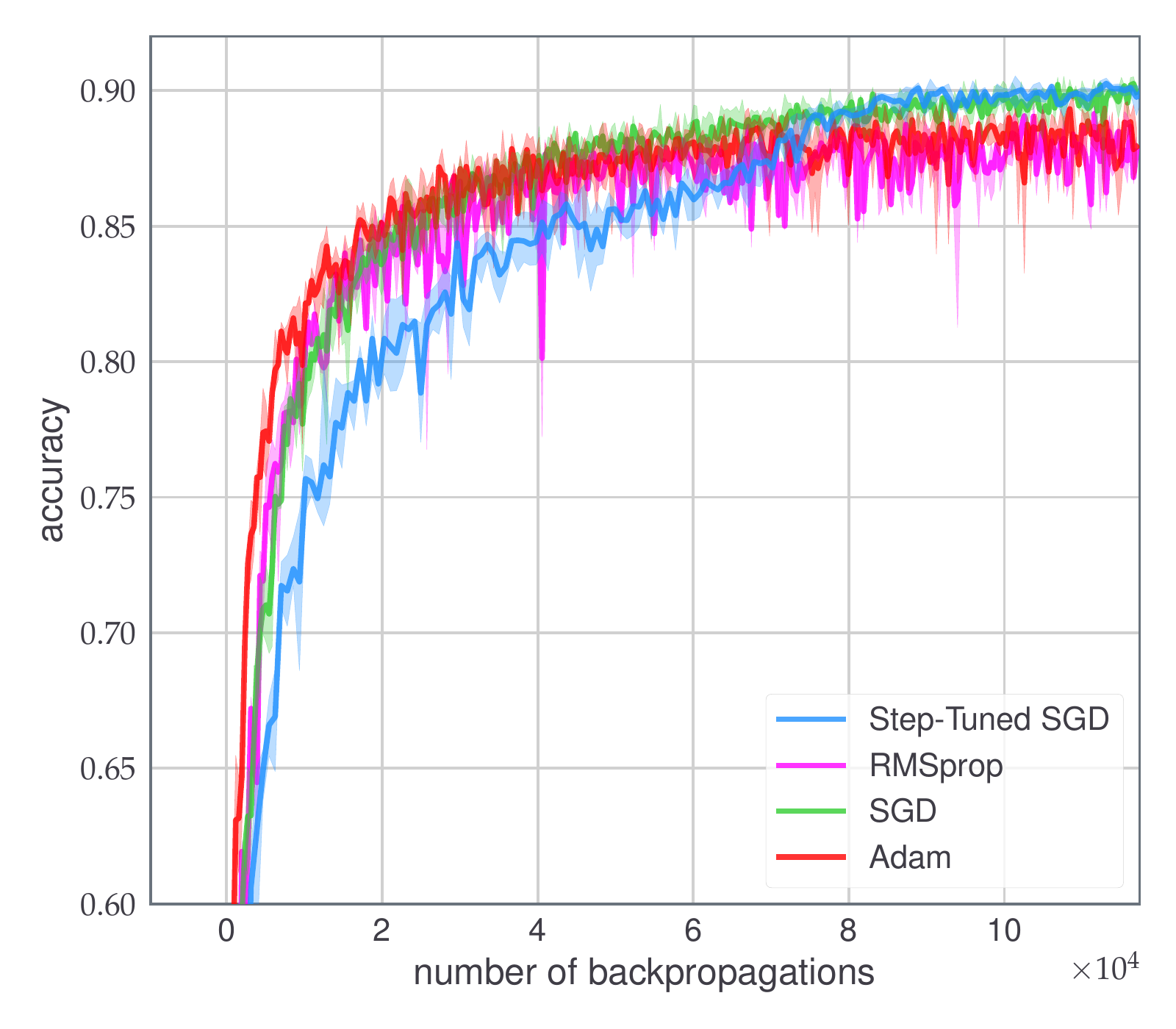}
            \end{minipage}%
            \begin{minipage}[c]{0.25\textwidth}
            \centering
            {\centering \scriptsize Problem (d): test error}
            \includegraphics[ width=\linewidth]{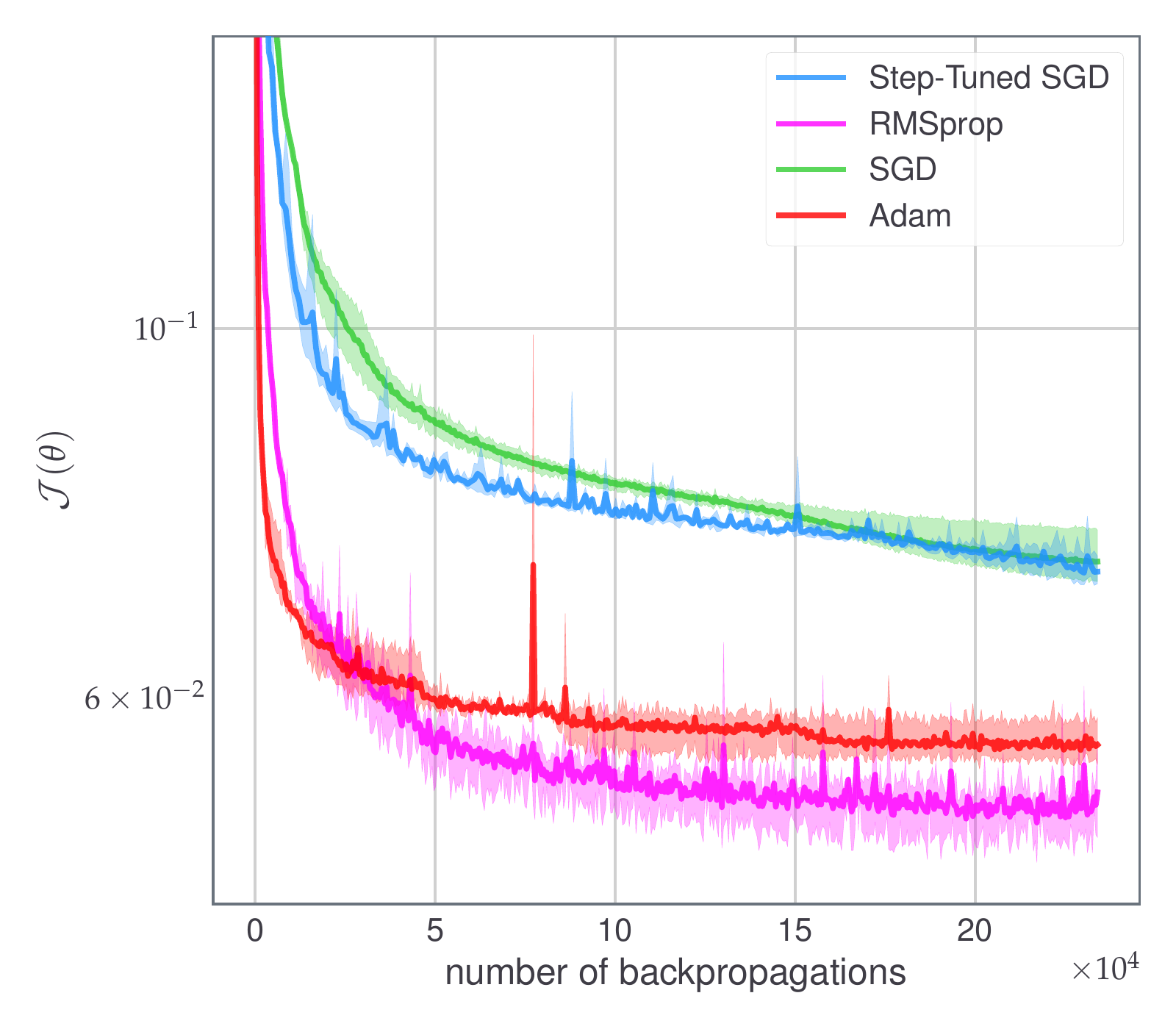}
            \end{minipage}%
            \caption{Classification of  CIFAR-10 and CIFAR-100  with ResNet-20 (left and middle-left respectively), CIFAR-10 with NiN (middle-right) and training of an auto-encoder on MNIST (right). This corresponds to Problems (a) to (d) specified in Table~\ref{tab::expsummary}. 
            Continuous lines: average values from 3 random initializations. Limits of shadow area: best and worst runs (in training loss). For fair comparison values are plotted against the number of gradient estimates computed (using back-propagation).
            \label{fig::probABCD}}
        \end{center}
    \end{figure}

    We describe the results for the two types of experiments, the comparative one to assess the quality of Step-Tuned SGD against concurrent optimization algorithms, and the other one to study the effect of changing the way mini-batches are used.
    \subsubsection{Comparison with standard methods.}
    The results for problems (a) to (d) are displayed on Figure~\ref{fig::probABCD}.
    For each problem we display the evolution of the values of the loss function and of the test accuracy during the training phase. We observe a recurrent behavior: during early training  Step-Tuned SGD behaves similarly than other methods, then there is a sudden drop of the loss (combined with an improvement in terms of test accuracy which we discuss below). As a result, Step-Tuned SGD achieves the best training performance among all algorithms on problems (a) and (b) and at least outperforms SGD in five of the six problems considered (result for Problems (e) and (f) are on Figure~\ref{fig::probEF}). The sudden drop observed is in accordance  with our preliminary observations in Figure \ref{fig::CD_MB}. We note that a similar drop and improved results are reported for SGD and ADAM when used with a manually enforced reduction of the learning rate see, e.g. \citet{he2015}. Our experiments show however that Step-Tuned SGD behaves similarly but in an automatic way: the \textit{drop down} is caused by the automatic fine-tuning the algorithm is designed to achieve and not by user-defined reduction of the step-size.
    
    We remark that in problem (d) ADAM and RMSprop are notably better than SGD and Step-Tuned SGD. This may be explained by their vector step-sizes (a scalar step-size for each coordinate of $\theta$) as auto-encoders are often ill-conditioned, making methods with scalar step-sizes less efficient. To conclude on these comparative experiments, in most cases Step-Tuned SGD represents a significant improvement compared to SGD. It also seems to be a good alternative to adaptive methods like RMSprop or ADAM especially on residual networks.  Note also that while stochastic second-order methods usually perform well mostly when combined with large mini-batches (hence with less-noisy gradients), we obtain satisfactory performances with mini-batches of standard sizes.
    
    In addition to efficient training performances (in terms of loss function values), Step-Tuned SGD generalizes well (as measured by test accuracy). For example Figure~\ref{fig::probABCD} shows a correlation between test accuracy and training loss. Conditions or explanations for when this happens are not fully understood to this day. Yet, SGD is often said to behave well with respect to this matter \citep{wilson2017marginal} and hence it is satisfactory to observe that Step-Tuned SGD seems to inherit this property.

    \begin{figure}[t]
        \begin{center}
            \begin{minipage}[c]{0.25\textwidth}
                \centering
                {\centering \scriptsize Problem (a): training error}
                \includegraphics[width=\linewidth]{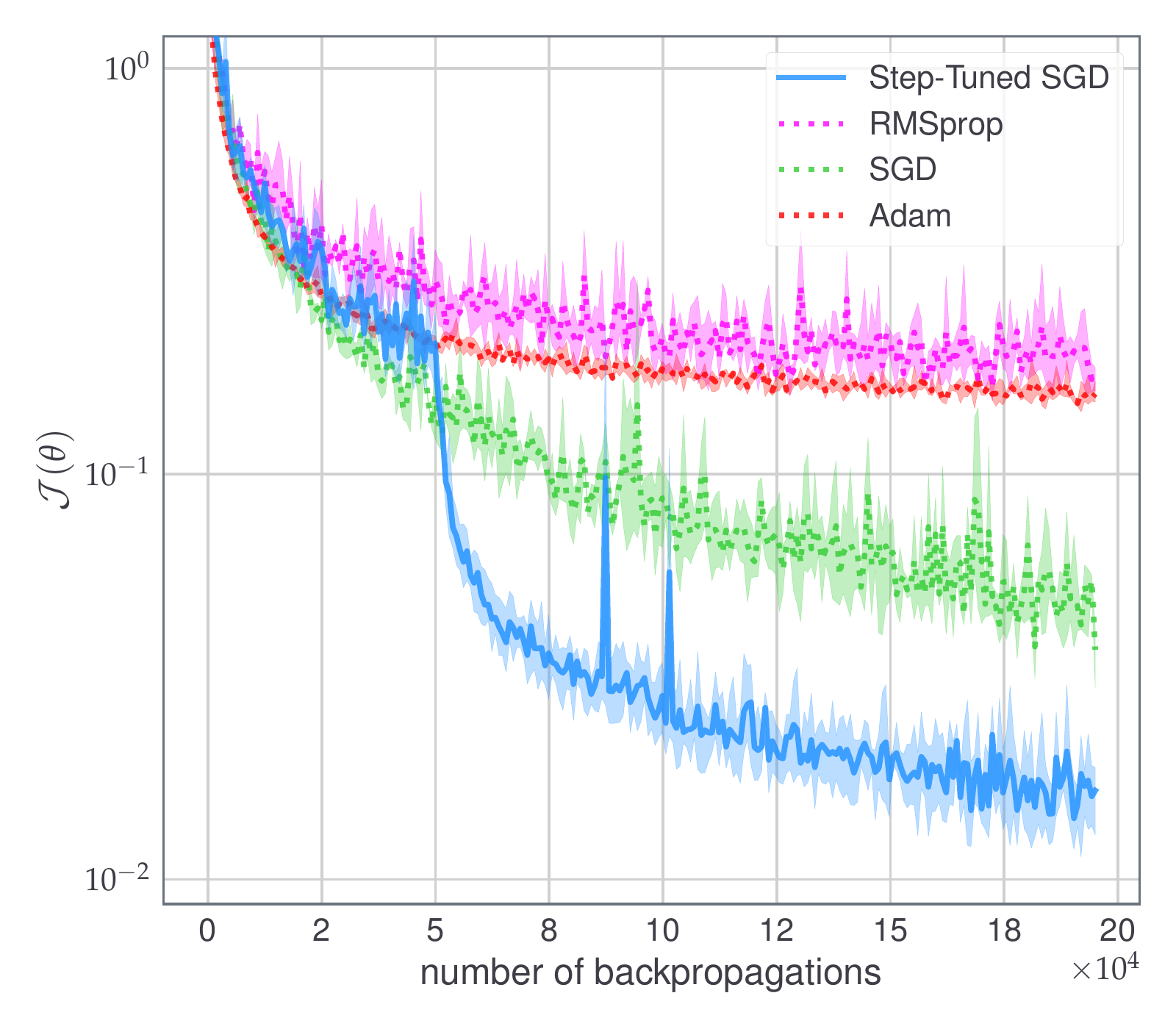}
            \end{minipage}%
            \begin{minipage}[c]{0.25\textwidth}
                \centering
                {\centering \scriptsize Problem (b): training error}
                \includegraphics[ width=\linewidth]{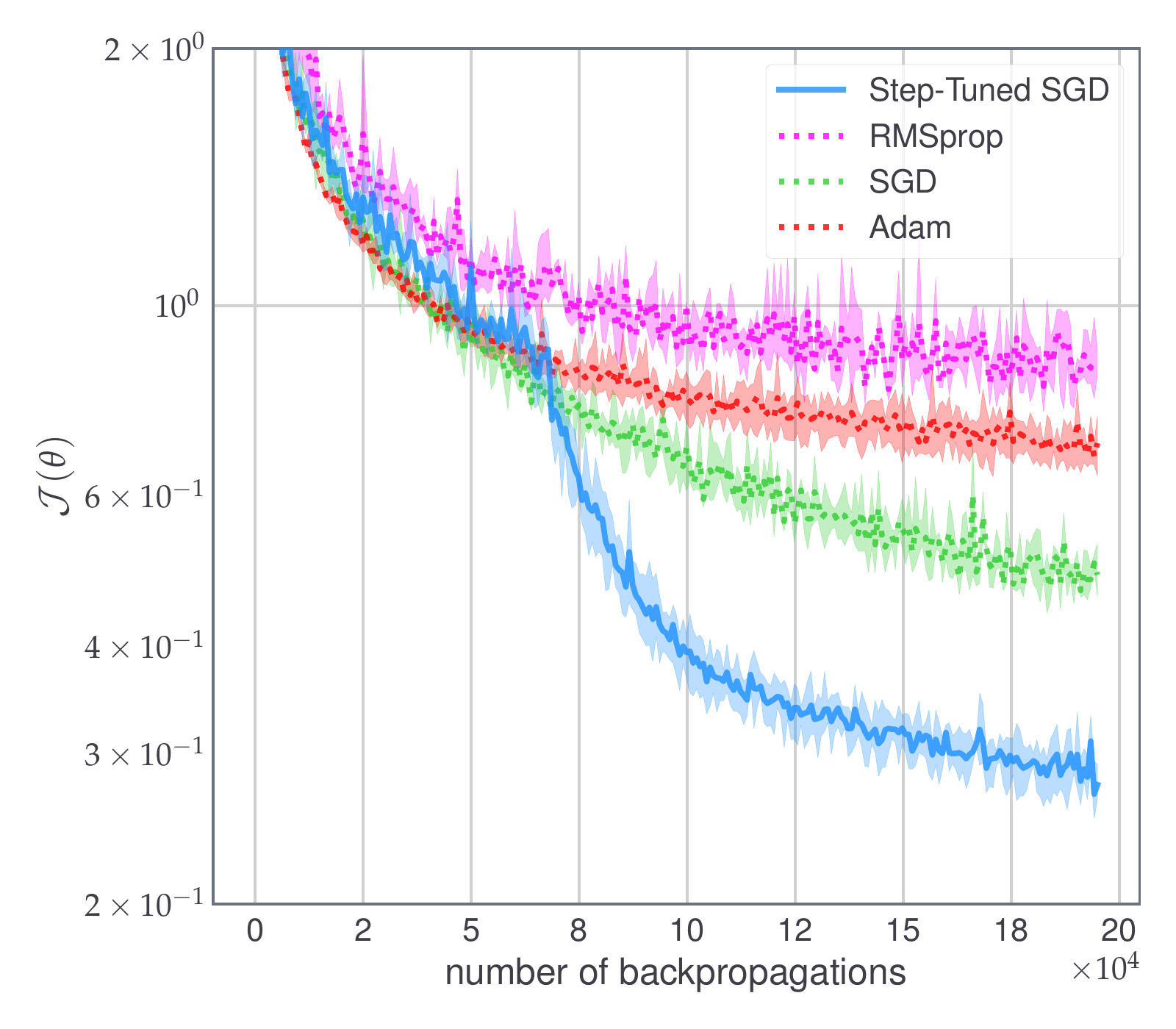}
            \end{minipage}%
            \begin{minipage}[c]{0.25\textwidth}
                \centering
                {\centering \scriptsize Problem (c): training error}
                \includegraphics[ width=\linewidth]{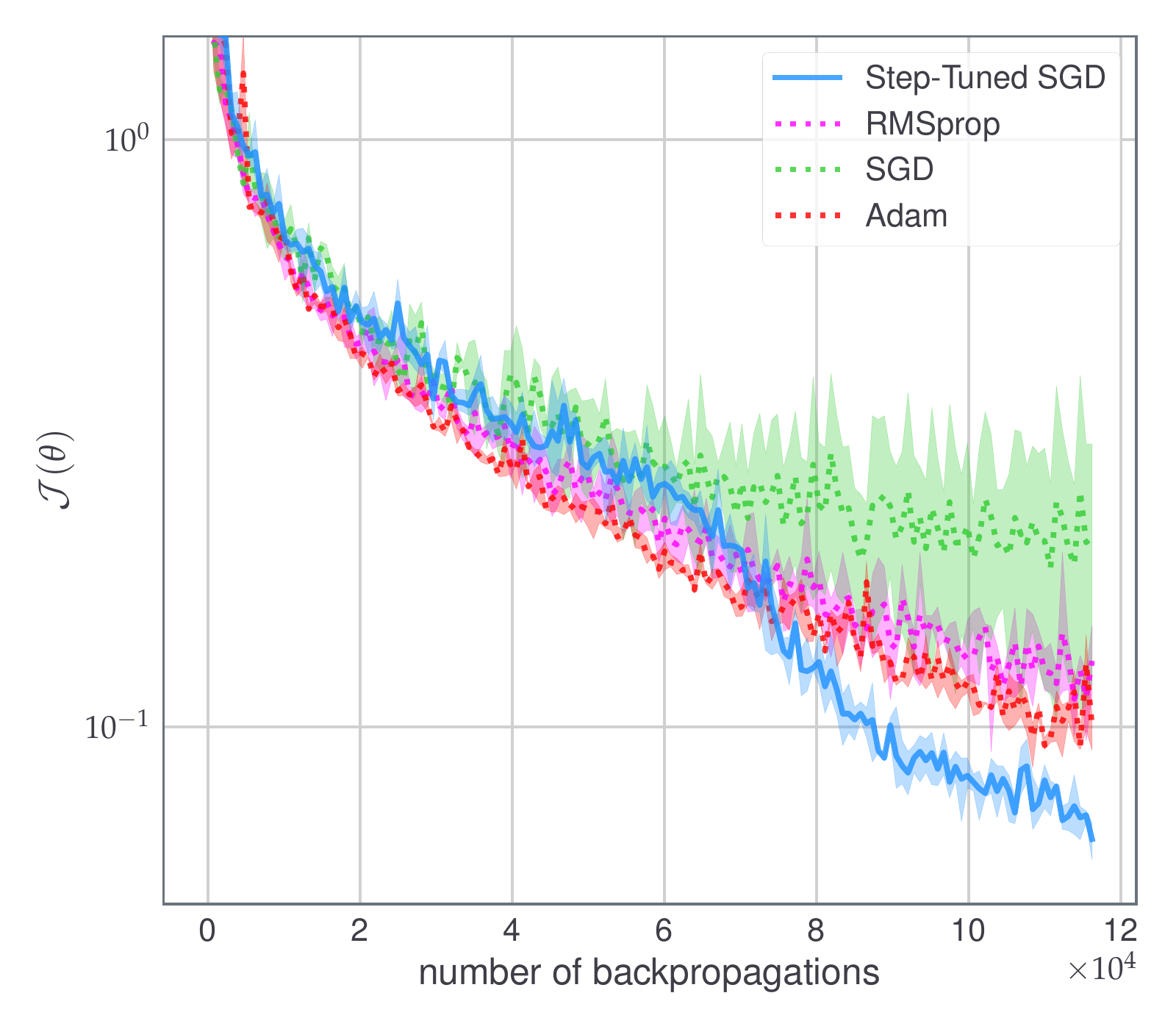}
            \end{minipage}%
            \begin{minipage}[c]{0.25\textwidth}
                \centering
                {\centering \scriptsize Problem (d): training error}
                \includegraphics[ width=\linewidth]{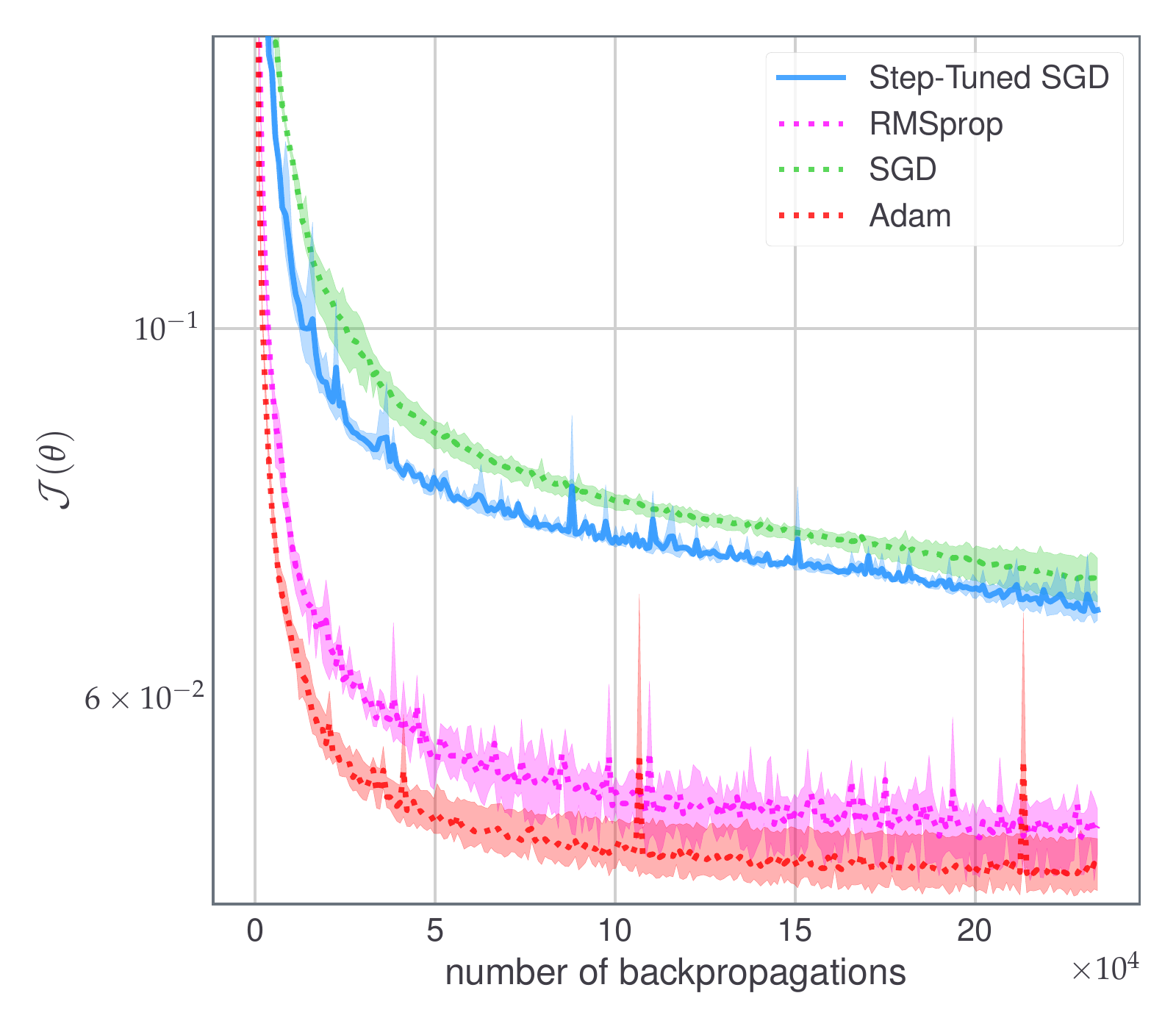}
            \end{minipage}%
            \caption{Experiment where each algorithms receives the same mini-batch for two consecutive iterations as in Algorithm~\ref{alg::Steptuned}. This allows to compare algorithms with respect to the number of data processed. The problems and the framework are the same as in Figure~\ref{fig::probABCD}.
            \label{fig::probABC2grads}}
        \end{center}
    \end{figure}
    \subsubsection{Effect of the mini-batch sub-sampling of Step-Tuned SGD.}
     The results are presented on Figure~\ref{fig::probABC2grads}. We observe that using each mini-batch twice usually reduces the performance of SGD, ADAM and RMSprop, except on problem (c) where it benefits the latter two in term of training error. Thus, on these problems, changing the way of using mini-batches is not the reason for the success of our method. On the contrary, it seems that our goal which was to obtain a fine-tuned step-size specifically for each iteration is clearly achieved, but processing data more slowly, like Step-Tuned SGD does, can sometimes impact the performances of the algorithm. 
     
    Arguably these results show that the need for using each mini-batch twice appears to be the main downside of Step-Tuned SGD. Thus in problems where mini-batches may be very different we should expect other methods to be more efficient as they process data twice faster. We actually remark that our method achieves its best results on networks where batch normalization (BatchNorm) is used, a technique that aims to normalize the inputs of neural networks \citep{ioffe2015batch}. Figure~\ref{fig::probEF} corroborates these observations: BatchNorm has a positive effect on Step-Tuned SGD.
    
     \begin{figure}[t]
        \begin{center}
            \begin{minipage}[c]{0.3\textwidth}
                \centering
                {\centering \scriptsize Without BatchNorm (Problem (e)): training error}
                \includegraphics[width=\linewidth]{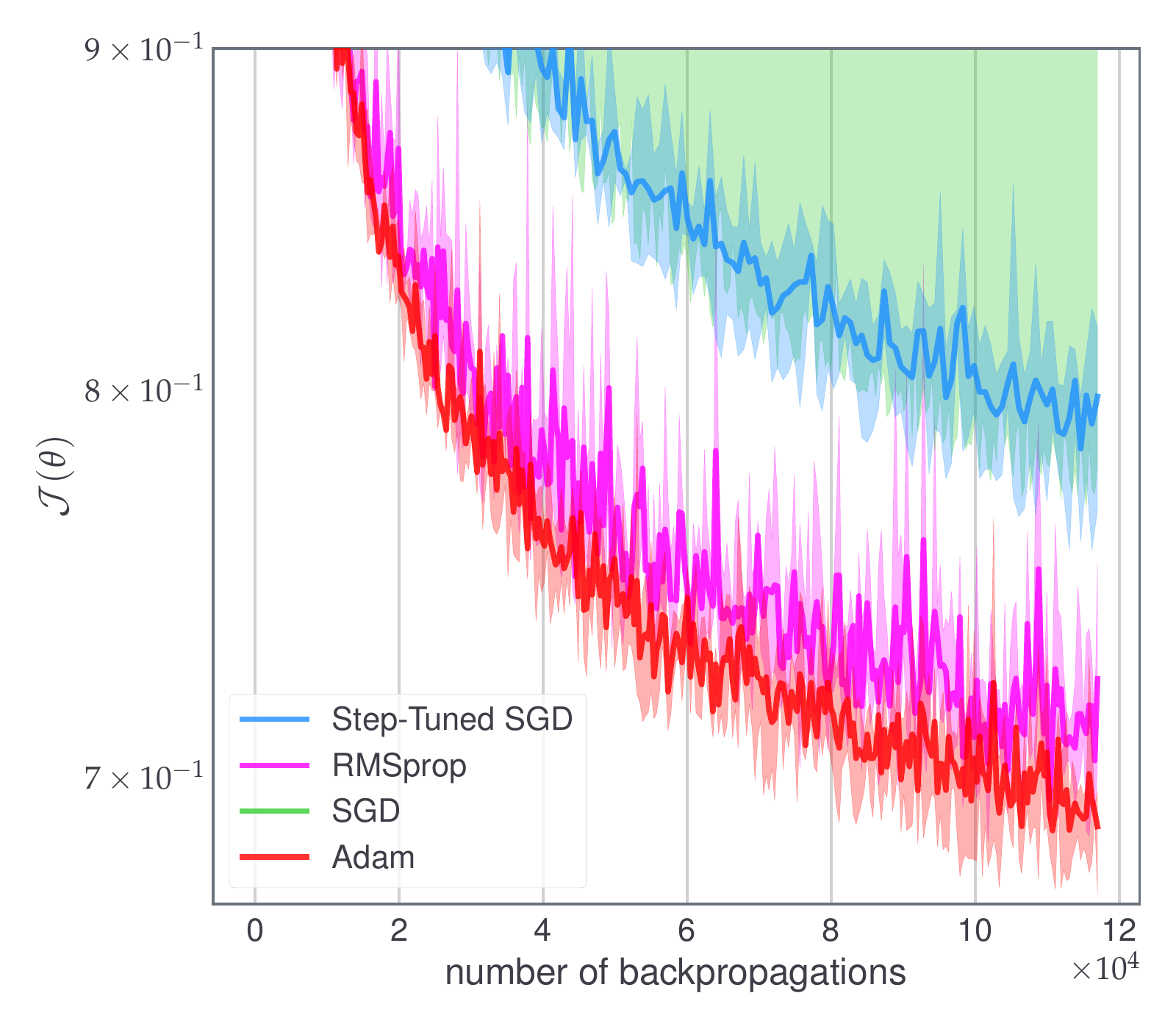}
            \end{minipage}%
            \begin{minipage}[c]{0.3\textwidth}
                \centering
                {\centering \scriptsize With BatchNorm (Problem (f)): training error}
                \includegraphics[ width=\linewidth]{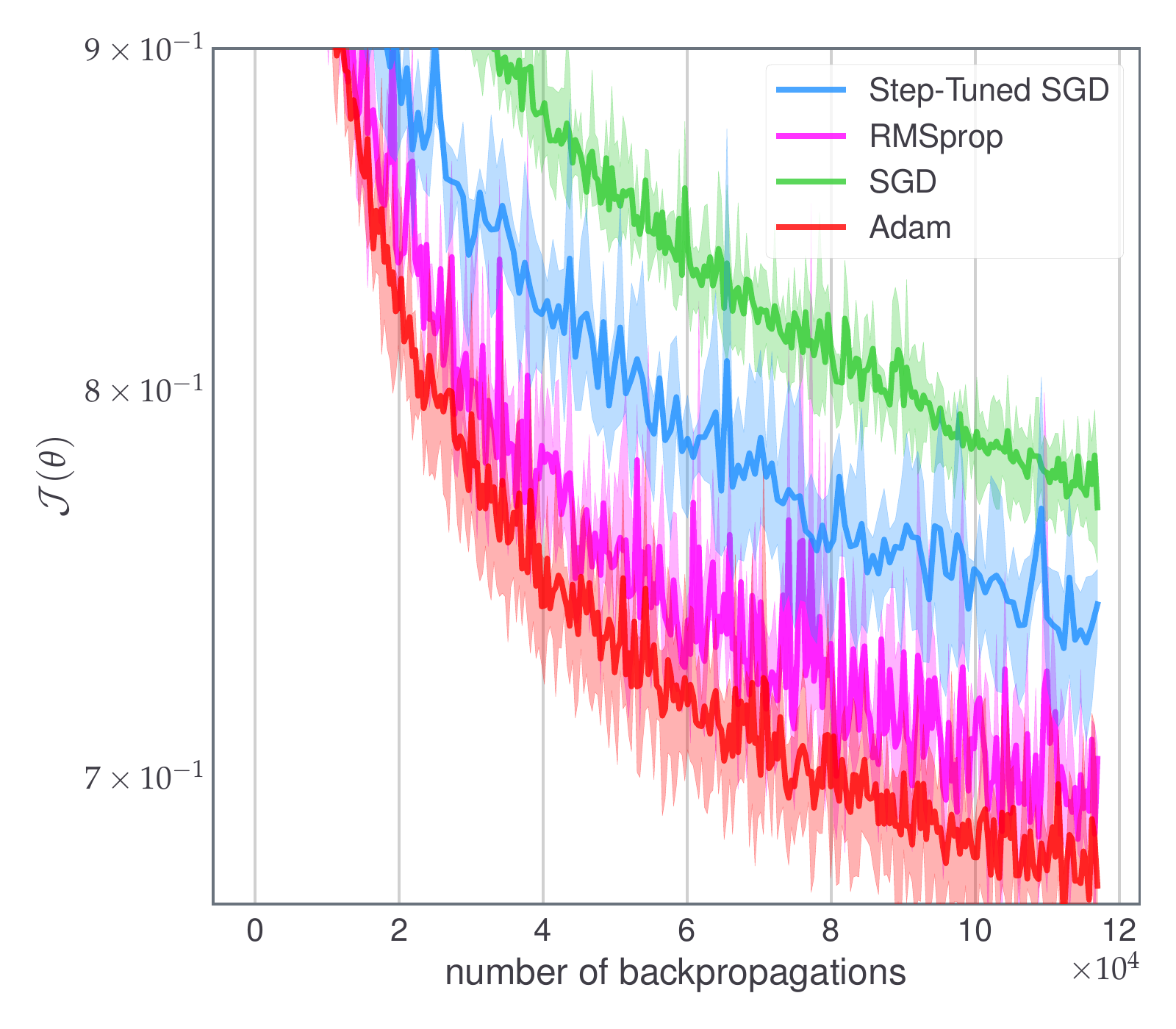}
            \end{minipage}%
            \caption{Classification of  CIFAR-10 with LeNet with and without batch normalization, corresponding to Problems (e) and (f) specified in Section~\ref{sec:expe} of the Supplementary. These experiments illustrate how batch normalization has a positive effect on Step-Tuned SGD.
            \label{fig::probEF}}
        \end{center}
    \end{figure}

    \section{Conclusion}
    We presented a new method to tune SGD's step-sizes for stochastic non-convex optimization within a  first-order computational framework.
    In addition to the new algorithm, we also presented a generic strategy (Section~ \ref{sec::deterministic} and \ref{sec::heuristic}) on how to use empirical and geometric considerations to address the major difficulty of preserving favorable behaviors of deterministic algorithms while dealing with mini-batches.  In particular, we tackled the problem of adapting the step-sizes to the local landscape of non-convex loss functions with noisy estimations.   For a computational cost similar to SGD, our method uses a step-size schedule changing every two iterations unlike other stochastic methods \`a la Barzilai-Borwein.
    Our algorithm comes with asymptotic convergence results and convergence rates.
    
    While our method does not alleviate hyper-parameter pre-tuning, it shows how an efficient \textit{automatic} fine-tuning of a \textit{simple} scalar step-size can improve the training of DNNs. Step-Tuned SGD processes data more slowly than other methods but by doing so manages to fine-tune step-sizes, leading to faster training in some DL problems with a typical sudden drop of the error rate at medium stages, especially on ResNets.

    \section*{Acknowledgements}
        The authors acknowledge the support of the European Research Council (ERC FACTORY-CoG-6681839), the Agence Nationale de la Recherche (ANR 3IA-ANITI, ANR-17-EURE-0010 CHESS, ANR-19-CE23-0017 MASDOL) and the Air Force Office of Scientific Research (FA9550-18-1-0226).
        
        Part of the numerical experiments were done using the OSIRIM platform of IRIT, supported by the CNRS, the FEDER, Région Occitanie and the French government (\url{http://osirim.irit.fr/site/en}). We thank the development teams of the following libraries that were used in the experiments: Python \citep{rossum1995python}, Numpy \citep{walt2011numpy}, Matplotlib \citep{hunter2007matplotlib}, PyTorch \citep{paszke2019pytorch}, and the PyTorch implementation of ResNets from \citet{Idelbayev18a}.

        We thank Emmanuel Soubies and Sixin Zhang for useful discussions and S\'ebastien Gadat for pointing out flaws in the original proof.
     \FloatBarrier
     \bibliographystyle{apalike}
     \bibliography{biblio.bib}

    \FloatBarrier
    \appendix 
    \section{Details about deep learning experiments}\label{sec:expe}
    In addition to the method described in Section~\ref{sec::settingDL}, we provide in Table~\ref{tab::expsummary} a summary of each problem considered.
        \begin{table}[h]
        \caption{Setting of the four different deep learning experiments.\label{tab::expsummary}}
        \ra{1.5}{}
        \renewcommand{\arraystretch}{1.6}
        \begin{tabular}{l|lll}\toprule
            &  \makecell{\textbf{Problem (a)}}  & \makecell{\textbf{Problem (b)}} & \makecell{\textbf{Problem (c)}} \\
            \midrule
            \makecell[l]{\textbf{Type}} & \makecell[c]{Classification}  & \makecell[c]{Classification}  & \makecell[c]{Classification}\\
            \makecell[l]{\textbf{Dataset}} & \makecell[c]{CIFAR-10}  & \makecell[c]{CIFAR-100 }& \makecell[c]{CIFAR-10} \\
            \makecell[l]{\textbf{Network}} & \makecell[c]{ResNet-20 \\(Residual)}  & \makecell[c]{ResNet-20 \\(Residual)} & \makecell[c]{Network-in-Network\\(Nested)} \\
            \makecell[l]{\textbf{BatchNorm}} & \makecell[c]{Yes} & \makecell[c]{Yes}& \makecell[c]{Yes} \\
            \makecell[l]{\textbf{Batch-size}} & \makecell[c]{$128$} & \makecell[c]{$128$} & \makecell[c]{$128$} \\
            \makecell[l]{\textbf{Activation functions}} & \makecell[c]{ReLU} & \makecell[c]{ReLU} & \makecell[c]{ELU} \\
            \makecell[l]{\textbf{Dissimilarity measure}} & \makecell[c]{Cross-entropy}  & \makecell[c]{Cross-entropy} & \makecell[c]{Cross-entropy}\\
            \makecell[l]{\textbf{Regularization}} &  \makecell[c]{$\lambda = 10^{-4}$} & \makecell[c]{$\lambda = 10^{-4}$} &  \makecell[c]{$\lambda = 10^{-4}$}\\
            \makecell[l]{\textbf{Grid-search}} & \makecell[c]{$50$ epochs} & \makecell[c]{$50$ epochs} & \makecell[c]{$30$ epochs}\\ 
            \makecell[l]{\textbf{Stop-criterion}} & \makecell[c]{$500$ epochs} & \makecell[c]{$500$ epochs} & \makecell[c]{$300$ epochs}\\
            \midrule
            &  \makecell{\textbf{Problem (d)}}& \makecell{\textbf{Problem (e)}} & \makecell{\textbf{Problem (f)}}\\
            \midrule
            \makecell[l]{\textbf{Type}} & \makecell[c]{Auto-encoder} & \makecell[c]{Classification}  & \makecell[c]{Classification}\\
            \makecell[l]{\textbf{Dataset}}  &\makecell[c]{MNIST} & \makecell[c]{CIFAR-10}  & \makecell[c]{CIFAR-10}\\
            \makecell[l]{\textbf{Network}}  & \makecell[c]{Auto-Encoder \\(Dense)} & \makecell[c]{LeNet\\(Convolutional)}  & \makecell[c]{LeNet \\(Convolutional)} \\
            \makecell[l]{\textbf{BatchNorm}} & \makecell[c]{No}& \makecell[c]{No} & \makecell[c]{Yes}\\
            \makecell[l]{\textbf{Batch-size}} & \makecell[c]{$128$} & \makecell[c]{$128$} & \makecell[c]{$128$}\\
            \makecell[l]{\textbf{Activation functions}} & \makecell[c]{SiLU} & \makecell[c]{ELU} & \makecell[c]{ELU}\\
            \makecell[l]{\textbf{Dissimilarity measure}} & \makecell[c]{Mean square} & \makecell[c]{Cross-entropy}  & \makecell[c]{Cross-entropy}\\
            \makecell[l]{\textbf{Regularization}} & \makecell[c]{$\lambda = 10^{-4}$} &  \makecell[c]{$\lambda = 10^{-4}$} & \makecell[c]{$\lambda = 10^{-4}$}\\
            \makecell[l]{\textbf{Grid-search}} & \makecell[c]{$50$ epochs} & \makecell[c]{$30$ epochs} & \makecell[c]{$30$ epochs}\\ 
            \makecell[l]{\textbf{Stop-criterion}} & \makecell[c]{$500$ epochs} & \makecell[c]{$300$ epochs} & \makecell[c]{$300$ epochs}\\
            \bottomrule
        \end{tabular}
        
    \end{table}
    
   In the DL experiments of Section~\ref{sec::numDL}, we display the training error and the test accuracy of each algorithm as a function of the number of stochastic gradient estimates computed. Due to their adaptive procedures, ADAM, RMSprop and Step-Tuned SGD have additional sub-routines in comparison to SGD. Thus, in Table~\ref{tab::wallclock} we additionally provide the wall-clock time per epoch of these methods relatively to SGD. Unlike the number of back-propagations performed, wall-clock time depends on many factors: the network and datasets considered, the computer used, and most importantly, the implementation. Regarding implementation, we would like to emphasize the fact that we used the versions of SGD, ADAM and RMSprop provided in \textit{PyTorch}, which are fully optimized (and in particular parallelized). Table~\ref{tab::wallclock} indicates that Step-Tuned SGD is slower than other adaptive methods for large networks but this is due to our non-parallel implementation. Actually on small networks (where the benefits of parallel computing is small), we observe that running Step-Tuned SGD for one epoch is actually faster than for SGD. 
   As a conclusion, the number of back-propagations is a more suitable metric for comparing the algorithms, and all methods considered require a single back-propagation per iteration.
   \begin{table}[t]
       \caption{Relative wall-clock time per epoch compared to SGD.\label{tab::wallclock}}
        \ra{1.5}{}
        \renewcommand{\arraystretch}{1.4}
        \begin{tabular}{l|cccccc}
        \toprule
            & Prob.(a) & Prob.(b) & Prob.(c) & Prob.(d) & Prob.(e) & Prob.(f)\\
            \midrule
            ADAM & 1.13 & 1.13 & 1.03 & 1.18 &1.04 & 1.00\\
            RMSprop & 1.06 & 1.08 & 1.02 &1.13 &1.00 & 1.01\\
            Step-Tuned SGD & 1.67 & 1.71 & 1.20 & 1.47 & 0.71 & 0.88\\
            \bottomrule
       \end{tabular}
   \end{table}

    \FloatBarrier
    \section{Proof of the theoretical results}\label{sec::supTheor}
    We state a lemma that we will use to prove Theorem~\ref{thm::mainres}.
    \subsection{Preliminary lemma}
    The result is the following.
    \begin{lemma}[{\citet[Proposition~2]{alber1998projected}}]\label{lem::gradconv}
        Let $(u_k)_{k\in\N}$ and $(v_k)_{k\in\N}$ two non-negative real sequences. Assume that $\sum_{k=0}^{+\infty} u_k v_k <+\infty$, and $\sum_{k=0}^{+\infty} v_k =+\infty$. If there exists a constant $C>0$ such that $\forall k\in\N, \vert u_{k+1} - u_k \vert\leq C v_k$, then $u_k\xrightarrow[k\to+\infty]{}0$.
    \end{lemma}
    \subsection{Proof of the main theorem}
    We can now prove Theorem~\ref{thm::mainres}.
    \begin{proof}[Proof of Theorem~\ref{thm::mainres}]
        We first clarify the random process induced by the draw of the mini-batches. Algorithm~\ref{alg::Steptuned} takes a sequence of mini-batches as input. This sequence is represented by the random variables $(\mb_k)_{k\in\N}$ as described in Section~\ref{sec::mini-batch}. Each of these random variables is independent of the others. In particular, for $k\in\N_{>0}$, $\mb_k$ is independent of the previous mini-batches $\mb_0,\ldots, \mb_{k-1}$. For convenience, we will denote $\umb_k = \left\{\mb_0,\ldots,\mb_k\right\}$, the mini-batches up to iteration $k$. Due to the randomness of the mini-batches, the algorithm is a random process as well. As such, $\tk$ is a random variable with a deterministic dependence on $\umb_{k-1}$ and is independent of $\mb_k$. However, $\tkh$ and $\mb_{k}$ are not independent. Similarly, we constructed $\gamma_k$ such that it is a random variable with a deterministic dependence on $\umb_{k-1}$, which is independent of $\mb_k$. This dependency structure will be crucial to derive and bound conditional expectations. Finally, we highlight the following important identity, for any $k\in \N_{>0}$,
        \begin{equation}
            \espcondTot{\nJ_{\mb_k}(\tk)} = \nJ(\tk).
        \end{equation}
        Indeed, the iterate $\tk$ is a deterministic function of $\umb_{k-1}$, so taking the expectation over $\mb_k$, which is independent of $\umb_{k-1}$, we recover the full gradient of $\J$ as the distribution of $\mb_k$ is the same as that of $\ms$ in Section \ref{sec::mini-batch}. Notice in addition that a similar identity does not hold for $\tkh$ (as it depends on $\mb_k$).
        
        We now provide estimates that will be used extensively in the rest of the proof.
        The gradient of the loss function $\nJ$ is locally Lipschitz continuous as $\J$ is twice continuously differentiable. By assumption, there exists a compact convex set $\mathsf{C}\subset\R^P$, such that with probability $1$, the sequence of iterates $(\tk)_{k\in\Nhalf}$ belongs to $\mathsf{C}$. Therefore, by local Lipschitz continuity, the restriction of $\nJ$ to $\mathsf{C}$ is Lipschitz continuous on $\mathsf{C}$. Similarly, each $\nJ_n$ is also Lipschitz continuous on $\mathsf{C}$. We denote by $L>0$ a Lipschitz constant common to each $\nJ_n$, $n=1,\ldots, N$. Notice that the Lipschitz continuity is preserved by averaging, in other words,
        \begin{equation}\label{eq::Lip}
            \forall \mb\subseteq\left\{1,\ldots,N\right\},\forall\psi_1,\psi_2\in\mathsf{C}, \quad\Vert \nabla\J_\mb(\psi_1) -\nabla\J_\mb(\psi_2) \Vert \leq  L\Vert \psi_1-\psi_2\Vert.
        \end{equation}
        In addition, using the continuity of the $\nabla\J_n$'s, there exists a constant $C_2>0$, such that,
        \begin{equation}\label{eq::C2}
            \forall \mb\subseteq\left\{1,\ldots,N\right\},\forall\psi\in\mathsf{C}, \quad\Vert \nJ_\mb(\psi)\Vert \leq  C_2.
        \end{equation}
        Finally, for a function $g:\R^P\to\R$ with $L$-Lipschitz continuous gradient, we recall the following inequality called descent lemma (see for example \citet[Proposition~A.24]{bertsekas1998nonlinear}). 
        For any $\theta\in\R^P$ and any $d\in \R^P$,
        \begin{equation}
            g(\theta+d) \leq g(\theta) + \langle \nabla g(\theta), d\rangle + \frac{L}{2}\Vert d \Vert^2.
        \end{equation}
        In our case since we only have the $L$-Lipschitz continuity of $\nJ$ on $\mathsf{C}$ which is convex, we have a similar bound for $\nJ$ on $\mathsf{C}$: for any $\theta\in\mathsf{C}$ and any $d\in \R^P$ such that $\theta+d\in\mathsf{C}$,
        \begin{equation}\label{eq::genDescent}
            \J(\theta+d) \leq \J(\theta) + \langle \nJ(\theta), d\rangle + \frac{L}{2}\Vert d \Vert^2.
        \end{equation}
        
        Let $\theta_0\in\R^P$ and let $(\tk)_{k\in\Nhalf}$ a sequence generated by Algorithm~\ref{alg::Steptuned} initialized at $\theta_0$. By assumption this sequence belongs to $\mathsf{C}$ almost surely.
        To simplify, for $k\in\N$, we denote $\eta_k = \alpha\gamma_k (k+1)^{-(1/2+\delta)}$.  Fix an iteration $k\in\N$, we can use \eqref{eq::genDescent} with $\theta = \theta_k$ and $d = -\eta_k\nJ_{\mb_{k}}(\theta_k)$, almost surely (with respect to the boundedness assumption),
        \begin{equation}\label{eq::AppliedDescent}
            \J(\tkh) \leq \J(\theta_k) - \eta_k \langle \nJ(\theta_k), \nJ_{\mb_{k}} (\theta_k) \rangle + \frac{\eta_k^2}{2}L \Vert \nJ_{\mb_{k}}(\theta_k)\Vert^2.
        \end{equation}
        Similarly with $\theta = \tkh$ and $d = -\eta_k\nJ_{\mb_{k}}(\tkh)$, almost surely,
        \begin{equation}\label{eq::AppliedDescent2}
            \J(\tkp) \leq \J(\tkh) - \eta_k \langle \nJ(\tkh), \nJ_{\mb_{k}} (\tkh) \rangle + \frac{\eta_k^2}{2}L \Vert \nJ_{\mb_{k}}(\tkh)\Vert^2.
        \end{equation}
        We combine \eqref{eq::AppliedDescent} and \eqref{eq::AppliedDescent2}, almost surely,
        \begin{align}\label{eq::AppliedDescent3}
        \begin{split}
            \J(\tkp) \leq \J(\tk) &- \eta_k \left(\langle \nJ(\theta_k), \nJ_{\mb_{k}} (\theta_k)\rangle + \langle \nJ(\tkh), \nJ_{\mb_{k}} (\tkh) \rangle  \right)\\&+ \frac{\eta_k^2}{2}L \left(\Vert \nJ_{\mb_{k}}(\theta_k)\Vert^2+ \Vert \nJ_{\mb_{k}}(\tkh)\Vert^2\right).
        \end{split}
        \end{align}
        Using the boundedness assumption and \eqref{eq::C2}, almost surely,
        \begin{equation}
         \Vert \nJ_{\mb_{k}}(\theta_k)\Vert^2 \leq C_2 \quad \text{and}\quad \Vert \nJ_{\mb_{k}}(\tkh)\Vert^2 \leq C_2.
        \end{equation}
        So almost surely,
        \begin{equation}\label{eq::AppliedDescent4}
            \J(\tkp) \leq \J(\tk) - \eta_k \left(\langle \nJ(\theta_k), \nJ_{\mb_{k}} (\theta_k)\rangle + \langle \nJ(\tkh), \nJ_{\mb_{k}} (\tkh) \rangle  \right) + \eta_k^2L C_2.
        \end{equation}
        Then, we take the conditional expectation of \eqref{eq::AppliedDescent4} over $\mb_k$ conditionally on $\umb_{k-1}$ (the mini-batches used up to iteration $k-1$), we have,
        \begin{align}\label{eq::espDescent}
        \begin{split}
            \espcondTot{\J(\tkp)} &\leq \espcondTot{\J(\tk)} + \espcondTot{\eta_k^2L C_2} \\&- \espcondTot{\eta_k \left(\langle \nJ(\theta_k), \nJ_{\mb_{k}} (\theta_k)\rangle + \langle \nJ(\tkh), \nJ_{\mb_{k}} (\tkh) \rangle  \right)}.
        \end{split}
        \end{align}
        As explained at the beginning of the proof, $\tk$ is a deterministic function of  $\umb_{k-1}$, thus, $\espcondTot{\J(\tk)} = \J(\tk)$. Similarly, by construction $\eta_k$ is independent of the current mini-batch $\mb_k$, it is a deterministic function of $\umb_{k-1}$. Hence, \eqref{eq::espDescent} reads,
        \begin{align}\label{eq::espDescent2}
        \begin{split}
            \espcondTot{\J(\tkp)} \leq& \J(\tk) + \eta_k^2L C_2 - \eta_k \langle \nJ(\theta_k), \espcondTot{\nJ_{\mb_{k}} (\theta_k)}\rangle\\& - \eta_k\espcondTot{\langle \nJ(\tkh), \nJ_{\mb_{k}} (\tkh) \rangle}  .
        \end{split}
        \end{align}
        Then, we use the fact that $\espcondTot{\nJ_{\mb_{k}} (\theta_k)} =\nJ (\theta_k) $. Overall, we obtain, 
        \begin{align}\label{eq::espDescent3}
        \begin{split}
            \espcondTot{\J(\tkp)} \leq& \J(\tk)+ \eta_k^2L C_2 -  \eta_k \Vert\nJ(\theta_k)\Vert^2 \\& - \eta_k\espcondTot{\langle \nJ(\tkh), \nJ_{\mb_{k}} (\tkh) \rangle}.
        \end{split}
        \end{align}
        We will now bound the last term of \eqref{eq::espDescent3}. First we write,
        \begin{align}\label{eq::peskyterm}
        \begin{split}
            &-\langle \nJ(\tkh), \nJ_{\mb_{k}} (\tkh) \rangle\\
            &=-\langle \nJ(\tkh), \nJ_{\mb_{k}} (\tkh)- \nJ_{\mb_{k}} (\tk)\rangle - \langle \nJ(\tkh), \nJ_{\mb_{k}} (\tk) \rangle.
        \end{split}
        \end{align}
        Using the Cauchy-Schwarz inequality, as well as \eqref{eq::Lip} and \eqref{eq::C2}, almost surely,
        \begin{align}
        \begin{split}
            |\langle \nJ(\tkh), \nJ_{\mb_{k}} (\tkh)- \nJ_{\mb_{k}} (\tk)\rangle| &\leq \Vert \nJ(\tkh) \Vert \Vert \nJ_{\mb_{k}} (\tkh)- \nJ_{\mb_{k}} (\tk)\Vert
            \\&\leq \Vert \nJ(\tkh) \Vert L\Vert \tkh-\tk\Vert
            \\& \leq \Vert \nJ(\tkh) \Vert L\Vert -\eta_k\nJ_{\mb_k}(\tk)\Vert
            \\& \leq LC_2^2\eta_k.
        \end{split}
        \end{align}
        Hence,
        \begin{align}\label{eq::peskyterm2}
        \begin{split}
            &-\langle \nJ(\tkh), \nJ_{\mb_{k}} (\tkh) \rangle \leq LC_2^2\eta_k - \langle \nJ(\tkh), \nJ_{\mb_{k}} (\tk) \rangle.
        \end{split}
        \end{align}
        We perform similar computations on the last term of \eqref{eq::peskyterm2}, almost surely
        \begin{align}\label{eq::peskyterm21}
        \begin{split}
            &-\langle \nJ(\tkh), \nJ_{\mb_{k}} (\tk) \rangle
            \\& = -\langle \nJ(\tkh)-\nJ(\tk), \nJ_{\mb_{k}} (\tk) \rangle - \langle \nJ(\tk), \nJ_{\mb_{k}} (\tk) \rangle
            \\& \leq \Vert \nJ(\tkh)-\nJ(\tk)\Vert \Vert \nJ_{\mb_{k}} (\tk) \Vert - \langle \nJ(\tk), \nJ_{\mb_{k}} (\tk) \rangle
            \\& \leq LC_2\Vert \tkh-\tk\Vert - \langle \nJ(\tk), \nJ_{\mb_{k}} (\tk) \rangle
            \\&\leq LC_2^2\eta_k - \langle \nJ(\tk), \nJ_{\mb_{k}} (\tk) \rangle.
        \end{split}
        \end{align}
        Finally we obtain by combining \eqref{eq::peskyterm}, \eqref{eq::peskyterm2} and \eqref{eq::peskyterm21}, almost surely,
        \begin{align}\label{eq::peskyterm3}
        \begin{split}
            &-\langle \nJ(\tkh), \nJ_{\mb_{k}} (\tkh) \rangle
            \leq 2LC_2^2\eta_k - \langle \nJ(\tk), \nJ_{\mb_{k}} (\tk) \rangle.
        \end{split}
        \end{align}
        Going back to the last term of \eqref{eq::espDescent3}, we have, taking the conditional expectation of \eqref{eq::peskyterm3}, almost surely
        \begin{align}\label{eq::peskyterm4}
        \begin{split}
            - &\eta_k\espcondTot{\langle \nJ(\tkh), \nJ_{\mb_{k}} (\tkh) \rangle}\\
            &\leq 2 L C_2^2\eta_k^2 - \eta_k \espcondTot{\langle \nJ(\tk), \nJ_{\mb_{k}} (\tk) \rangle}
            \\&\leq 2 L C_2^2\eta_k^2 - \eta_k \langle \nJ(\tk), \espcondTot{\nJ_{\mb_{k}} (\tk)}\rangle
            \\&= 2 L C_2^2\eta_k^2 - \eta_k  \Vert\nJ(\tk)\Vert^2.
        \end{split}
        \end{align}
        In the end we obtain, for an arbitrary iteration $k\in \N$, almost surely
        \begin{align}\label{eq::espDescent4}
        \begin{split}
            \espcondTot{\J(\tkp)} \leq& \J(\tk) - 2\eta_k \Vert\nJ(\theta_k)\Vert^2 + \eta_k^2L(C_2+2C_2^2).
        \end{split}
        \end{align}
        To simplify we assume that $\tilde M\geq\nu$ (otherwise set $\tilde M = \max(\tilde M,\nu)$). We use the fact that, $\eta_k\in[\frac{\alpha\tilde m}{(k+1)^{1/2+\delta}},\frac{\alpha\tilde M}{(k+1)^{1/2+\delta}}]$, to obtain almost surely,
        \begin{align}\label{eq::espcond3}
        \begin{split}
            \espcondTot{\J(\tkp)} \leq& \J(\tk) - 2\frac{\alpha\tilde{m}}{(k+1)^{1/2+\delta}} \Vert\nJ(\theta_k)\Vert^2 + \frac{\alpha^2\tilde{M}^2}{(k+1)^{1+2\delta}}L(C_2+2C_2^2).
        \end{split}
        \end{align}
    Since by assumption, the last term is summable, we can now invoke Robbins-Siegmund convergence theorem \cite{robbins1971convergence} to obtain that, almost surely, $(\J(\tk))_{k\in\N}$ converges and,
    \begin{equation}\label{eq::cvsumgrad}
        \sum_{k=0}^{+\infty}\frac{1}{(k+1)^{1/2+\delta}}\Vert \nJ(\theta_k) \Vert^2 < + \infty.
    \end{equation}
    Since $\sum_{k=0}^{+\infty}\frac{1}{(k+1)^{1/2+\delta}}=+\infty$, this implies at least that almost surely, 
    \begin{equation}
        \liminf_{k\to\infty}\Vert \nJ(\theta_k) \Vert^2=0.
    \end{equation}
    To prove that in addition $\displaystyle\lim_{k\to\infty}\Vert\nJ(\theta_k)\Vert^2 = 0$, we will use Lemma~\ref{lem::gradconv} with $u_k = \Vert\nJ(\theta_{k})\Vert^2$ and $v_k = \frac{1}{(k+1)^{1/2+\delta}}$, for all $k\in\N$. So we need to prove that there exists $C_3>0$ such that $\vert u_{k+1} - u_k\vert \leq C_3 v_k$. 
    To do so, we use the $L$-Lipschitz continuity of the gradients on $\mathsf{C}$, triangle inequalities and \eqref{eq::C2}. It holds, almost surely, for all $k \in \N$
    \begin{align}
    \begin{split}
        &\left\vert\Vert\nJ(\theta_{k+1})\Vert^2-\Vert\nJ(\theta_{k})\Vert^2\right\vert\\
        =& \;\left(\;\Vert\nJ(\theta_{k+1})\Vert+ \Vert\nJ(\theta_{k})\Vert\;\right)\;\times\;\left\vert\;\Vert\;\nJ(\theta_{k+1})\Vert- \Vert\nJ(\theta_{k})\;\Vert\;\right\vert\\
        \leq& 2C_2 \left\vert\Vert\nJ(\theta_{k+1})\Vert- \Vert\nJ(\theta_{k})\Vert\right\vert\\
        \leq& 2C_2 \Vert\nJ(\theta_{k+1})-\nJ(\theta_{k})\Vert\\
        \leq& 2C_2 L \Vert\theta_{k+1}-\theta_{k}\Vert\\
        \leq& 2C_2 L \left\Vert-\eta_k\nJ_{\mb_{k}}(\theta_k)-\eta_k\nJ_{\mb_{k}}(\tkh)\right\Vert\\
        \leq& 2C_2 L\frac{\alpha\tilde{M}}{(k+1)^{1/2+\delta}} \Vert\nJ_{\mb_{k}}(\theta_k)+\nJ_{\mb_{k}}(\tkh)\Vert\\
        \leq& 4C_2^2 L\frac{\alpha\tilde{M}}{(k+1)^{1/2+\delta}}.
    \end{split}
    \end{align}
    So taking $C_3 =4C_2^2 L\alpha\tilde{M} $, by Lemma~\ref{lem::gradconv}, almost surely, $\lim_{k\to+\infty} \Vert\nJ(\theta_{k})\Vert^2=0$. This concludes the almost sure convergence proof.
    
    As for the rate, consider  the expectation of  \eqref{eq::espcond3} (with respect to the random variables $(\mb_k)_{k\in\N}$).
    The tower property of the conditional expectation gives,
    $$ \mathbb{E}[\mathbb{E}[\J(\theta_{k+1})|\umb_{k-1}]]=\esp{\J(\theta_{k+1})},$$
     so we obtain, for all $k\in\N$,
    \begin{align}\label{eq::esptot}
    \begin{split}
    2\frac{\alpha\tilde{m}}{(k+1)^{1/2+\delta}}\esp{\Vert \nJ(\theta_k)\Vert^2} \leq& \esp{\J(\theta_k)}  -  \esp{\J(\theta_{k+1})} + \frac{\alpha^2\tilde{M}^2}{(k+1)^{1+2\delta}}L(C_2+2C_2^2).
    \end{split}
    \end{align}

    Then for $K\geq 1$, we sum from $0$ to $K-1$,
    \begin{align}\label{eq::sumesp1}
    \begin{split}
     \sum_{k=0}^{K-1}2\frac{\alpha\tilde{m}}{(k+1)^{1/2+\delta}}&\esp{\Vert \nJ(\theta_k)\Vert^2}\\
     &\leq \sum_{k=0}^{K-1}\esp{\J(\theta_k)} -\sum_{k=0}^{K-1}\esp{\J(\theta_{k+1})} + \sum_{k=0}^{K-1}\frac{\alpha^2\tilde{M}^2}{(k+1)^{1+2\delta}}L(C_2+2C_2^2)\\
     &=\J(\theta_0) - \mathbb{E}\left[\J(\theta_{K})\right] + \sum_{k=0}^{K-1}\frac{\alpha^2\tilde{M}^2}{(k+1)^{1+2\delta}}L(C_2+2C_2^2)\\
      &\leq  \J(\theta_0) - \inf_{\psi\in\R^P}\J(\psi) + \sum_{k=0}^{K-1}\frac{\alpha^2\tilde{M}^2}{(k+1)^{1+2\delta}}L(C_2+2C_2^2),\
    \end{split}
    \end{align}
   The right-hand side is finite, so there is a constant $C_4>0$ such that for any $K\in\N$, it holds,
    \begin{multline}\label{eq::rate}
        C_4\geq \sum_{k=0}^K \frac{1}{(k+1)^{1/2+\delta}} \esp{\Vert\nJ(\theta_{k})\Vert^2} \geq \min_{k\in\left\{1,\ldots,K\right\}}\esp{\Vert\nJ(\theta_{k})\Vert^2}\sum_{k=0}^K \frac{1}{(k+1)^{1/2+\delta}} \\ \geq \left(K+1\right)^{1/2-\delta}\min_{k\in\left\{1,\ldots,K\right\}} \esp{\Vert\nJ(\theta_{k})\Vert^2},
    \end{multline}
    and we obtain the rate.
\end{proof}

    \subsection{Proof of the corollary}
    Before proving the corollary we recall the following result.
    \begin{lemma}\label{lem::Lipschitzbounded}
        Let $g:\R^P\to\R$ a $L$-Lipschitz continuous and differentiable function. Then $\nabla g$ is uniformly bounded on $\R^P$.
    \end{lemma}
    We can now prove the corollary.
    \begin{proof}[Proof of Corollary~\ref{cor::globLipschitz}]
    The proof is very similar to the one of Theorem~\ref{thm::mainres}. Denote $L$ the Lipschitz constant of $\nJ$. Then, the descent lemma \eqref{eq::AppliedDescent} holds surely.
    Furthermore, since for all $n\in\{1,\ldots,N\}$, each $\J_n$ is Lipschitz, so is $\J$, and globally Lipschitz functions have uniformly bounded gradients so $\nJ$ has bounded gradient. This is enough to obtain \eqref{eq::espcond3}. Similarly, at iteration $k\in\N$, $\esp{\Vert \nJ_{\mb_{k}} (\theta_k)\Vert}$ is also uniformly bounded. Overall these arguments allows to follow the lines of the proof of Theorem~\ref{thm::mainres} and the same conclusions follow by repeating the same arguments. %
\end{proof}

    \section{Details on the synthetic experiments \label{sec::numCD}}
    We detail the non-convex regression problem that we presented in Figure~\ref{fig::bbvsnc} and~\ref{fig::CD_MB}.
    Given a matrix $A\in\R^{N \times P}$ and a vector $b\in\R^N$, denote $A_n$ the n-\textit{th} line of $A$. The problem consists in minimizing a loss function of the form,
    \begin{equation}\label{eq::CDJ}
        \theta\in \R^P\mapsto \J(\theta) = \frac{1}{N}\sum_{n}^{N} \phi(A_n^T\theta-b_n),
    \end{equation} where the non-convexity comes from the function $t\in\R\mapsto\phi(t) = t^2/(1+t^2)$. For more details on the initialization of $A$ and $b$ we refer to \citet{carmon2017convex} where this problem is initially proposed. 
    In the experiments of Figure~\ref{fig::CD_MB}, the mini-batch approximation was made by selecting a subset of the lines of $A$, which amounts to computing only a few terms of the full sum in \eqref{eq::CDJ}. We used $N=500$, $P=30$ and mini-batches of size $50$. 
    
    In the deterministic setting we ran each algorithm during $250$ iterations and selected the hyper-parameters of each algorithm such that they achieved $\vert\J(\theta)-\J^\star\vert<10^{-1}$ as fast as possible. 
    In the mini-batch experiments we ran each algorithm during $250$ epochs and selected the hyper-parameters that yielded the smallest value of $\J(\theta)$ after $50$ epochs.
    \section{Description of auxiliary algorithms}\label{supp::otherAlgs}
    We precise the heuristic algorithms used in Figure~\ref{fig::CD_MB} and discussed in Section~\ref{sec::heuristic}. Note that the step-size in Algorithm~\ref{alg::ExpectedGV} is equivalent to \ref{eq::algoExpe2} but is written differently to avoid storing an additional gradient estimate.
    \begin{algorithm}[H]
        \caption{Stochastic-GV SGD}
        \begin{algorithmic}[1]
            
            \STATE{\bfseries Input:} $\alpha>0$, $\nu>0$
            \STATE{\bfseries Input:} $\tilde{m}>0$, $\tilde{M}>0$, $\delta\in(0
            ,1/2)$
            \STATE{\bfseries Initialize} $\theta_0\in\R^P$, $\gamma_0=1$
            \STATE{\bfseries Draw} mini-batches $(\mb_k)_{k\in\N}$ independently and uniformly at random with replacement.
            
            \STATE $\theta_1 = \theta_0 - \alpha\gamma_0\nabla\J_{\mb_0}(\theta_0)$
            \FOR{$k= 1,\ldots$}

            \STATE $\dtk=\theta_{k}-\theta_{k-1}$
            \STATE $\dgk^{\mathrm{naive}}=\nJ_{\mb_{k}}(\tk)-\nJ_{\mb_{k-1}}(\tkm)$
            \IF{$\langle \dgk^{\mathrm{naive}},\dtkh\rangle>0$}
            \STATE $\gamma_{k}=\frac{\Vert\dtk\Vert^2}{\langle \dgk^{\mathrm{naive}},\dtk\rangle}$
            \ELSE
            \STATE $\gamma_{k}=\nu$
            \ENDIF 
            \STATE $\gamma_{k}=\min(\max(\gamma_{k},\tilde{m}),\tilde{M})$
            
            \STATE $\theta_{k+1}=\theta_{k} - \frac{\alpha}{(k+1)^{1/2+\delta}}\gamma_{k} \nJ_{\mb_k}(\theta_{k})$
            \ENDFOR
        \end{algorithmic}
    \end{algorithm}
    \begin{algorithm}[H]
        \caption{Exact-GV SGD}
        \begin{algorithmic}[1]
            
            \STATE{\bfseries Input:} $\alpha>0$, $\nu>0$
            \STATE{\bfseries Input:} $\tilde{m}>0$, $\tilde{M}>0$, $\delta\in(0
            ,1/2)$
            \STATE{\bfseries Initialize} $\theta_0\in\R^P$, $\gamma_0=1$
            \STATE{\bfseries Draw} mini-batches $(\mb_k)_{k\in\N}$ independently and uniformly at random with replacement.
            \STATE $\theta_1 = \theta_0 - \alpha\gamma_0\nabla\J_{\mb_0}(\theta_0)$
            \FOR{$k=1,\ldots$}

            \STATE $\dtk=\theta_{k}-\theta_{k-1}$
            \STATE $G_{k}=\nJ(\tk)-\nJ(\tkm)$
            \IF{$\langle {G}_{k},\dtkh\rangle>0$}
            \STATE $\gamma_{k}=\frac{\Vert\dtk\Vert^2}{\langle {G}_{k},\dtk\rangle}$
            \ELSE
            \STATE $\gamma_{k}=\nu$
            \ENDIF 
            \STATE $\gamma_{k}=\min(\max(\gamma_{k},\tilde{m}),\tilde{M})$
            
            \STATE $\theta_{k+1}=\theta_{k} - \frac{\alpha}{(k+1)^{1/2+\delta}}\gamma_{k} \nJ_{\mb_k}(\theta_{k})$
            \ENDFOR
        \end{algorithmic}
    \end{algorithm}
    \begin{algorithm}[H]
        \caption{Expected-GV SGD\label{alg::ExpectedGV}}
        \begin{algorithmic}[1]
            
            \STATE{\bfseries Input:} $\alpha>0$, $\nu>0$
            \STATE{\bfseries Input:} $\tilde{m}>0$, $\tilde{M}>0$, $\delta\in(0
            ,1/2)$
            \STATE{\bfseries Initialize} $\theta_0\in\R^P$, $\gamma_0=1$
            \STATE{\bfseries Draw} mini-batches $(\mb_k)_{k\in\N}$ independently and uniformly at random with replacement.
            \STATE $\theta_1 = \theta_0 - \alpha\gamma_0\nabla\J_{\mb_0}(\theta_0)$
            \FOR{$k=1,\ldots$}

            \STATE $\dtk=\theta_{k}-\theta_{k-1}$
            \STATE $G_{k}=-\frac{\alpha}{(k-1)^{1/2+\delta}}\gamma_{k-1}\esp{\CJBkm(\theta_{k-1})}$
            \IF{$\langle {G}_{k},\dtkh\rangle>0$}
            \STATE $\gamma_{k}=\frac{\Vert\dtk\Vert^2}{\langle {G}_{k},\dtk\rangle}$
            \ELSE
            \STATE $\gamma_{k}=\nu$
            \ENDIF 
            \STATE $\gamma_{k}=\min(\max(\gamma_{k},\tilde{m}),\tilde{M})$
            
            \STATE $\theta_{k+1}=\theta_{k} - \frac{\alpha}{(k+1)^{1/2+\delta}}\gamma_{k} \nJ_{\mb_k}(\theta_{k})$
            \ENDFOR
        \end{algorithmic}
    \end{algorithm}

\end{document}